\documentclass[letterpaper]{article} %
\usepackage{aaai25}  %
\usepackage{times}  %
\usepackage{helvet}  %
\usepackage{courier}  %
\usepackage[hyphens]{url}  %
\usepackage{graphicx} %
\usepackage{natbib}  %
\usepackage{caption} %
\usepackage[]{hyperref}
\hypersetup{
}

\title{On the Relationship Between Monotone and Squared Probabilistic Circuits}
\author{
    Benjie Wang, 
    Guy Van den Broeck
}
\affiliations{
    University of California, Los Angeles\\

    benjiewang@cs.ucla.edu, guyvdb@cs.ucla.edu
}

\usepackage{mathtools} %
\usepackage{booktabs} %
\usepackage{tikz} %

\newcommand{\circuit}{\mathcal{C}}

\newcommand{\vars}{\bm{V}}
\newcommand{\varsval}{\bm{v}}
\newcommand{\var}{V}
\newcommand{\varsubset}{\bm{W}}
\newcommand{\complex}{\mathbb{C}}

\newcommand{\weight}{\theta}

\newcommand{\node}{n}

\newcommand{\inputs}{\textnormal{in}}
\newcommand{\pcfunc}{f}
\newcommand{\probfunc}{p}

\newcommand{\partitionfunc}{Z}

\newcommand{\scope}{\textnormal{sc}}

\newcommand{\reals}{\mathbb{R}}

\newcommand{\twolatents}{\bm{W}}

\newcommand{\onelatents}{\bm{U}}

\newcommand{\twolatentsval}{\bm{w}}
\newcommand{\onelatentsval}{\bm{u}}

\newcommand{\modelname}{Inception}
\newcommand{\modelshort}{IncPC}

\usepackage{bm}
\usepackage{dsfont}
\usepackage{amsthm}
\usepackage{amsmath}
\usepackage{amsfonts}
\newtheorem{theorem}{Theorem}
\newtheorem{definition}{Definition}

\newtheorem{lemma}{Lemma}
\usepackage{thm-restate}
\usepackage{stmaryrd}

\usepackage[ruled, vlined, linesnumbered, noend,]{algorithm2e}
\DontPrintSemicolon
\SetKwInput{KwInput}{Input}
\SetKwInput{KwOutput}{Output}

\usepackage{subcaption}

\tikzstyle{sum1}=[draw,circle,inner sep=1pt,fill=yellow!30]
\tikzstyle{sum2}=[draw,circle,inner sep=1pt,fill=cyan!30]
\tikzstyle{prod}=[draw,circle,inner sep=1pt,fill=purple!30]
\tikzstyle{input}=[draw,circle,inner sep=1pt,minimum size=12pt]
\tikzstyle{lit}=[inner sep=2pt]

\begin{document}

\maketitle

\begin{abstract}
  Probabilistic circuits are a unifying representation of functions as computation graphs of weighted sums and products. Their primary application is in probabilistic modeling, where circuits with non-negative weights (\emph{monotone} circuits) can be used to represent and learn density/mass functions, with tractable marginal inference. %
    Recently, it was proposed to instead represent densities as the \emph{square} of the circuit function (\emph{squared} circuits); this allows the use of negative weights while retaining tractability, and can be exponentially more expressive efficient than monotone circuits. Unfortunately, we show the reverse also holds, meaning that monotone circuits and squared circuits are incomparable in general. %
    This raises the question of whether we can reconcile, and indeed improve upon the two modeling approaches. We answer in the positive by proposing \modelname{} PCs, a novel type of circuit that naturally encompasses both monotone circuits and squared circuits as special cases, and employs complex parameters. %
    Empirically, we validate that \modelname{} PCs can outperform both monotone and squared circuits on a range of tabular and image datasets.
\end{abstract}

\begin{links}
     \link{Code}{https://github.com/wangben88/InceptionPCs}
\end{links}

\section{Introduction}

The philosophy of tractable probabilistic modeling advocates for modeling high-dimensional probability distributions using model architectures that support exact and efficient inference for various probabilistic queries \emph{by-design}. This property makes them extremely useful for probabilistic reasoning; for example, tractable models have found wide-ranging applications from enhancing and controlling intractable generative models \citep{ZhangICML23,LiuICLR24}, to neuro-symbolic AI \citep{AhmedNeurIPS22,AhmedNeurIPS23}, to causal discovery and inference \citep{wang2022trust,wang2023compositional}. 

The \emph{lingua franca} for tractable models is the probabilistic circuits framework (PC) \citep{ProbCirc20}, which specify functions using \emph{computation graphs} of sums and products, unifying many previous tractable models such as arithmetic circuits \cite{darwiche2003differential}, sum-product networks \cite{poon2011sum} and cutset networks \cite{rahman2014cutset}.  
The standard approach to modeling is to directly represent the probability distribution (density/mass function) using a PC. In order to enforce non-negativity of the PC output, one typically restricts to the PC to have non-negative parameters; these are known as \emph{monotone} PCs \citep{darwiche2003differential,poon2011sum}. However, recent works have also shown that there exist many tractable classes of probability distributions that provably cannot be expressed in this way \citep{zhang2020relationship,yu2023characteristic,broadrick2024polynomial}.

This motivates the development of new approaches for \emph{practically constructing} and \emph{learning} generalized tractable models. To this end, \citet{LoconteICLR24} recently proposed \emph{squared circuits}, where the probability distribution is defined to be (proportional to) the \emph{square} of the circuit function. In this formulation, a PC can employ real (possibly negative parameters) while still defining a valid probability distribution. This was shown theoretically to lead to an exponential advantage in expressive efficiency compared to monotone PCs, for certain classes of distributions.

In this work, we reexamine monotone and squared (structured-decomposable) PCs, and show that they are incomparable in general: either can be exponentially more expressive efficient than the other. Motivated by this observation, we draw on the latent variable interpretation of PCs \citep{peharz2016latent} and show an elegant connection between the two types of circuits; namely, that they simply correspond to summing the latent variables \emph{outside or inside the square}, i.e. a deep sum-of-squares or square-of-sums. 
By combining these two types of latent variables, we introduce a novel class of tractable models representing deep \emph{sum-of-square-of-sums}, which we call \modelname{} PCs (\modelshort), strictly generalizing and extending (structured-decomposable) monotone and squared PCs.

We further investigate learning Inception PCs effectively at scale. Following recent trends in PC learning \citep{peharz2020ratspn,peharz2020einsum,liu2021tractable,mari2023unifying}, we design an efficient tensorized implementation of Inception PCs that subsumes tensorized (structured-decomposable) monotone and squared PCs. %
Empirical results validate the improved expressive efficiency of Inception PCs on density estimation benchmarks. %

Our contributions can be summarized as follows: 

\begin{itemize}
    \item Theoretically, we show that monotone circuits can be \emph{exponentially} smaller than squared circuits for a simple class of distributions, meaning that 
    monotone and squared circuits are incomparable in general in terms of expressive efficiency (Section \ref{sec:exp_efficiency});
    \item We analyze both monotone and squared circuits from a latent variable perspective, showing that they can be interpreted as deep \emph{sums-of-squares} and \emph{squares-of-sums} respectively. We then propose a novel class of tractable circuits, Inception PCs (\modelshort), which generalize and extend both types of circuits as a deep \emph{sum-of-square-of-sums} and can employ complex parameters (Section \ref{sec:model}). %
    \item We propose an efficient tensorized architecture for Inception PCs, enabling their application to large-scale datasets (Section \ref{sec:tensorized});
    \item Empirically, we demonstrate that (i) combining both types of latents, together with complex parameters, can improve performance; and (ii) \modelname{} PCs can outperform monotone PCs and squared PCs on challenging benchmarks including downscaled versions of ImageNet, even when normalized for computation time (Section \ref{sec:experiments}).
\end{itemize}
In concurrent work, \citet{loconte2024sumsquarescircuits} proved an exponential separation between monotone and squared circuits (i.e., our Theorem \ref{thm:monotone_succinct}) using a similar family of separating functions. They also proposed two new model classes, SOCS and $\mu$SOCS, which generalize squared circuits. We discuss in detail connections with our Inception PCs in Appendix \ref{apx:other_models}.

\section{Preliminaries}

\paragraph{Notation} We use capital letters to denote variables and lowercase to denote their assignments/values (e.g. $X, x$). We use boldface (e.g. $\bm{X}, \bm{x}$) to denote sets of variables/assignments.

Probabilistic circuits are \emph{computation graphs} representing functions constructed by hierarchical compositions of weighted sums and products. %

\begin{definition}[Probabilistic Circuit]
A probabilistic circuit $\circuit$ over a set of variables $\vars$ is a rooted DAG consisting of three types of nodes $\node$: input, product and sum nodes. Each input node $\node$ is a leaf encoding a function $\pcfunc_{\node}: \varsubset \to \reals$ for some $\varsubset \subseteq \vars$, and for each internal (product or sum) node $\node$, denoting the set of children (i.e. nodes $\node'$ for which $\node \rightarrow \node'$) by $\inputs(\node)$, we define:
\begin{equation}
    \pcfunc_{\node} = \begin{cases}
        \prod_{\node_i \in \inputs(\node)} \pcfunc_{\node} & \text{if } \node \text{ is product;} \\
        \sum_{\node_i \in \inputs(\node)} \weight_{\node, \node_i} \pcfunc_{\node_i} & \text{if } \node \text{ is sum.} \\
    \end{cases}
\end{equation}
where each sum node has a set of weights $\{\weight_{\node, \node_i}\}_{\node_i \in \inputs(\node)}$ with $\weight_{\node, \node_i} \in \reals$. Each node $\node$ thus encodes a function over a set of variables $\scope(\node)$, which we call its \emph{scope}; this is given by $\scope(\node) = \bigcup_{\node_i \in \inputs(\node)} \scope(\node_i)$ for internal nodes. The function encoded by the circuit $\pcfunc_{\circuit}$ is the function encoded by its root node. The size of a probabilistic circuit $|\circuit|$ is defined to be the number of edges in its DAG.

\end{definition}

In this paper, we will assume that sum and product nodes alternate; this is without loss of generality as this property can be enforced on a PC at most a linear increase in size. 
A key feature of the sum-product structure of probabilistic circuits is that they allow for efficient (linear-time) computation of marginals, for example the partition function $\partitionfunc = \sum_{\varsval} \pcfunc(\varsval)$\footnote{alternatively, $\int$ in the case of continuous variables}, if they are \emph{smooth} and \emph{decomposable}:

\begin{definition}[Smoothness, Decomposability]
    A probabilistic circuit is \emph{smooth} if for every sum node $\node$, its inputs $\node_i$ have the same scope. A probabilistic circuit is \emph{decomposable} if for every product node $\node$, its inputs have disjoint scope.
\end{definition}

We will also need a stronger version of decomposability  that enables circuits to be multiplied together efficiently \citep{pipatsrisawat08compilation,vergari2021compositional}:

\begin{definition}[Structured Decomposability]
    A smooth and decomposable probabilistic circuit is \emph{structured-decomposable} if any two product nodes $\node, \node'$ with the same scope decompose in the same way.
\end{definition}

Structured-decomposability is commonly enforced when learning PCs from data \citep{liu2021tractable} as well as when compiling a circuit from another model such as a Bayesian network \citep{choi2013compiling}.

\section{Expressive Efficiency of Monotone and Squared Structured-Decomposable Circuits} \label{sec:exp_efficiency}

One of the primary applications of probabilistic circuits is as a tractable representation of probability distributions. As such, we typically require the function output of the circuit to be a non-negative real. The usual way to achieve this is to enforce non-negativity of the weights and input functions:
 
\begin{definition}[Monotone PC]
    A probabilistic circuit is \emph{monotone} if all weights are non-negative reals, and all input functions map to the non-negative reals.
\end{definition}

Given a \emph{monotone} PC $\circuit$, one can define a probability distribution $\probfunc_1(\vars) := \frac{\pcfunc_{\circuit}(\vars)}{\partitionfunc_{\circuit}}$ where $\partitionfunc_{\circuit}$ is the partition function of the PC. However, this is not the only way to construct a non-negative function. In \citet{LoconteICLR24}, it was proposed to instead use $\pcfunc_{\circuit}$ to represent a \emph{real} (i.e. possibly negative) function, by allowing for real weights/input functions; this can then be squared to obtain a non-negative function. That is, we define $\probfunc_2(\vars) := \frac{\pcfunc_{\circuit}(\vars)^2}{\sum_{\varsval}\pcfunc_{\circuit}(\varsval)^2}$.

In order for $\sum_{\varsval}\pcfunc_{\circuit}(\varsval)^2$ to be tractable to compute, a sufficient condition is for the circuit $\circuit$ to be structured-decomposable; one can then explicitly construct a smooth and (structured-)decomposable circuit $\circuit^2$ such that $\pcfunc_{\circuit^2}(\vars) = \pcfunc_{\circuit}(\vars)^2$  of size and in time $O(|\circuit|^2)$ \citep{vergari2021compositional}. Then we have that $\probfunc_2(\vars) = \frac{\pcfunc_{\circuit^2}(\vars)}{\partitionfunc_{\circuit^2}}$, i.e. the distribution induced by the PC $\circuit^2$. Crucially, the circuit $\circuit^2$ is not necessarily monotone; squaring thus provides an alternative means of constructing PCs that represent non-negative functions. 
In fact, it is known that squared real PCs can be exponentially more expressive efficient than structured-decomposable monotone PCs for representing probability distributions:

\begin{theorem} \label{thm:sq_succinct} \citep{LoconteICLR24} 
    There exists a class of non-negative functions $p(\bm{V})$ such that there exist structured-decomposable PCs $\circuit$ with $p(\bm{V}) = f_{\circuit}(\bm{V})^2$ of size polynomial in $|\bm{V}|$, but the smallest structured-decomposable monotone PC $\circuit'$ such that $p(\bm{V}) = f_{\circuit'}(\bm{V})$ has size $2^{\Omega(|\bm{V}|)}$.
\end{theorem}

However, we now show that, in fact, the other direction also holds: monotone PCs can also be exponentially more expressive efficient than squared (real) PCs.

\begin{restatable}{theorem}{thmMonoSuccinct} \label{thm:monotone_succinct}
    There exists a class of non-negative functions $p(\bm{V})$, such that there exist monotone structured-decomposable PCs $\circuit$ with $p(\bm{V}) = f_{\circuit}(\bm{V})$ of size polynomial in $|\bm{V}|$, but the smallest structured-decomposable PC $\circuit'$ such that $p(\bm{V}) = f_{\circuit'}(\bm{V})^2$ has size $2^{\Omega(|\bm{V}|)}$. %
\end{restatable}
\begin{proof}
    (Sketch) The function class we use to separate the circuit classes is $p(\bm{V}) = \sum_{i=0}^{d-1} 2^{i} \mathds{1}_{V_i = 1} + 1$, i.e. a function that outputs the integer encoded in binary by the variables ($+ 1$). This can easily be represented as a structured-decomposable monotone circuit of linear size. We show a lower bound on the size of structured-decomposable circuits computing any square root $F(\bm{V})$ of this function, by (i) reducing to bounding the rank of a matrix representation of $F$ w.r.t. balanced partitions $(\bm{X}, \bm{Y})$ of $\bm{V}$ \citep{martens2014expressive}, using a standard techinque from communication complexity; (ii) lower bounding the number of distinct prime square roots in the matrix; (iii) showing the existence of a sufficiently large submatrix with full rank, thus lower bounding the rank of the original matrix \cite{fawzi2015positive}.
\end{proof}

This is perhaps surprising, as squaring PCs generate structured PCs with possibly negative weights, suggesting that they should be more general than monotone structured PCs. The key point is that not all circuits that represent a positive function (not even all monotone structured ones) can be generated by squaring. 
Taken together, these results are somewhat unsatisfying, as we know that there are some distributions better represented by an unsquared monotone PC, and some by a squared real PC. In the next section, we will investigate how to reconcile these different approaches to specifying probability distributions.

\section{Towards a Unified Model for Deep Sums-of-Squares-of-Sums} \label{sec:model}

In this section, we investigate the relationship between monotone and squared circuits in depth. Firstly, we show how to introduce complex parameterizations to squared circuits (Prop \ref{prop:conjugation}). Secondly, we provide a new interpretation of monotone and squared circuits as different ways of marginalizing out latent variables (Figure \ref{fig:circuit1}); and introduce a new tractable model, \modelname{} PCs, that combines the two approaches (Theorem \ref{thm:tract_model}).

\subsection{Complex Parameters} \label{sec:complex}
We begin by noting that, beyond simply negative parameters, one can also allow for \emph{complex} weights and input functions\footnote{Circuits with complex-valued input functions have previously been proposed by \citet{yu2023characteristic} in the context of representing characteristic functions (an alternative representation of probability distributions).}, i.e. take values in the field $\complex$. Then, to ensure the non-negativity of the squared circuit, we multiply a circuit with its \emph{complex conjugate}. That is:
\begin{align*}
    \probfunc_{2}(\vars) = \frac{|\pcfunc_{\circuit}(\vars)|^2}{\sum_{\varsval}|\pcfunc_{\circuit}(\varsval)|^2} = \frac{\overline{\pcfunc_{\circuit}(\vars)} \pcfunc_{\circuit}(\vars)}{\sum_{\varsval} \overline{\pcfunc_{\circuit}(\varsval)}\pcfunc_{\circuit}(\varsval)}
\end{align*}

As complex conjugation is a field isomorphism of $\complex$, taking a complex conjugate of a circuit is as straightforward as taking the complex conjugate of each weight and input function, retaining the same DAG as the original circuit. This allows us to efficiently compute $p_2(\bm{V})$ as a (smooth and structured decomposable) circuit:
\begin{restatable}[Tractability of Complex Conjugation]{proposition}{propConjugation} \label{prop:conjugation}
    Given a smooth and decomposable circuit $\circuit$, it is possible to compute a smooth and decomposable circuit $\overline{\circuit}$ such that $\pcfunc_{\overline{\circuit}}(\vars) = \overline{\pcfunc_{\circuit}(\vars)}$ of size and in time $O(|\circuit|)$. Further, if $\circuit$ is structured decomposable, then it is possible to compute a smooth and structured decomposable $\circuit^2$ s.t. $\pcfunc_{\circuit^2}(\vars) = \overline{\pcfunc_{\circuit}(\vars)} \pcfunc_{\circuit}(\vars)$ of size and in time $O(|\circuit|^2)$.
\end{restatable}

\subsection{Deep Sums-of-Squares-of-Sums: A Latent Variable Interpretation}

\begin{figure*}[t]
    \centering
    \begin{subfigure}{0.26\linewidth}
    \centering
        \begin{tikzpicture}[semithick]
\tikzstyle{sum1}=[draw,circle,inner sep=1.2pt, cyan, ultra thick]
        \tikzstyle{prod}=[draw,circle,inner sep=1.2pt,orange, ultra thick]
        \tikzstyle{input}=[draw, circle,inner sep=2pt,black,ultra thick]
        \begin{scope}
            \node[sum1] (sum1) at (0,-0.5) {$\bm{+}$};
            \node[prod] (prod1) at (-.8,-1.5) {$\bm{\times}$}; 
            \node[prod] (prod2) at (.8,-1.5) {$\bm{\times}$}; 

            \node[] (dummy1) at (-1.4, -2.2) {};
            \node[] (dummy2) at (-0.8, -2.2) {};
            \node[] (dummy3) at (0.2, -2.2) {};
            \node[] (dummy4) at (0.8, -2.2) {};

            \node[input] (Z0) at (-0.2, -2.2) {$\bm{\circ}$};%
            \node[input] (Z1) at (1.4, -2.2) {$\bm{\circ}$};%

            \node[] (Z0label) at (-0.2, -2.6) {\scriptsize $\llbracket Z = 0 \rrbracket$};
            \node[] (Z1label) at (1.4, -2.6) {\scriptsize $\llbracket Z = 1 \rrbracket$};

            \draw[->] (sum1) -- (prod1) node[midway, above left] {\scriptsize $3 + 2i$};
            \draw[->] (sum1) -- (prod2) node[midway, above right] {\scriptsize $1 - i$};
            
            \foreach \s/\t in {prod1/dummy1,prod1/dummy2,prod1/Z0,prod2/dummy3,prod2/dummy4,prod2/Z1} {
                \draw[dashed] (\s) edge (\t);
            }\
            \foreach \s/\t in {prod1/Z0,prod2/Z1} {
                \draw[-] (\s) edge (\t);
            }
        \end{scope}
    \end{tikzpicture}
        \caption{LVI of sum node $\node$}
        \label{fig:lvi}
    \end{subfigure}
    \begin{subfigure}{0.26\linewidth}
    \centering
        \begin{tikzpicture}[semithick]
\tikzstyle{sum1}=[draw,circle,inner sep=1.2pt, cyan, ultra thick]
        \tikzstyle{prod}=[draw,circle,inner sep=1.2pt,orange, ultra thick]
        \tikzstyle{input}=[draw, circle,inner sep=2pt,black,ultra thick]
        \begin{scope}
            \node[sum1] (sum1) at (0,-0.5) {$\bm{+}$};
            \node[prod] (prod1) at (-.8,-1.5) {$\bm{\times}$}; 
            \node[prod] (prod2) at (.8,-1.5) {$\bm{\times}$}; 

            \node[] (dummy1) at (-1.4, -2.2) {};
            \node[] (dummy2) at (-0.8, -2.2) {};
            \node[] (dummy3) at (0.2, -2.2) {};
            \node[] (dummy4) at (0.8, -2.2) {};

            \draw[->] (sum1) -- (prod1) node[midway, above left] {\scriptsize $13$};
            \draw[->] (sum1) -- (prod2) node[midway, above right] {\scriptsize $2$};
            
            \foreach \s/\t in {prod1/dummy1,prod1/dummy2,prod2/dummy3,prod2/dummy4} {
                \draw[dashed] (\s) edge (\t);
            }\
        \end{scope}
    \end{tikzpicture}
        \caption{Computing $\sum_{Z} |\pcfunc_{\node}(\vars, Z)|^2$}
        \label{fig:square_before_sum}
    \end{subfigure}
    \begin{subfigure}{0.44\linewidth}
    \centering
        \begin{tikzpicture}[semithick]
\tikzstyle{sum1}=[draw,circle,inner sep=1.2pt, cyan, ultra thick]
        \tikzstyle{prod}=[draw,circle,inner sep=1.2pt,orange, ultra thick]
        \tikzstyle{input}=[draw, circle,inner sep=2pt,black,ultra thick]
        \begin{scope}
            \node[sum1] (sum1) at (0,-0.5) {$\bm{+}$};
            \node[prod] (prod1) at (-3,-1.5) {$\bm{\times}$}; 
            \node[prod] (prod2) at (-0.8,-1.5) {$\bm{\times}$}; 

    \node[prod] (prod3) at (.8,-1.5) {$\bm{\times}$}; 
                \node[prod] (prod4) at( 3,-1.5) {$\bm{\times}$}; 

            \node[] (dummy1) at (-3.6, -2.2) {};
            \node[] (dummy2) at (-3.0, -2.2) {};
            \node[] (dummy3) at (-1.4, -2.2) {};
            \node[] (dummy4) at (-0.8, -2.2) {};
            \node[] (dummy5) at (0.2, -2.2) {};
            \node[] (dummy6) at (0.8, -2.2) {};
            \node[] (dummy7) at (2.4, -2.2) {};
            \node[] (dummy8) at (3.0, -2.2) {};

            \draw[->] (sum1) -- (prod1) node[midway, above left] {\scriptsize $13$};
            \draw[->] (sum1) -- (prod2) node[midway, left] {\scriptsize $5 - 5i$};
            \draw[->] (sum1) -- (prod3) node[midway, right] {\scriptsize $5 + 5i$};
            \draw[->] (sum1) -- (prod4) node[midway, above right] {\scriptsize $2$};
            
            \foreach \s/\t in {prod1/dummy1,prod1/dummy2,prod2/dummy3,prod2/dummy4,prod3/dummy5,prod3/dummy6,prod4/dummy7,prod4/dummy8} {
                \draw[dashed] (\s) edge (\t);
            }\
        \end{scope}
    \end{tikzpicture}
        \caption{Computing $|\sum_{Z} \pcfunc_{\node}(\vars, Z)|^2$}
        \label{fig:sum_before_square}
    \end{subfigure}
    \caption{Latent variable interpretation for squaring PCs. The sum node in Figure \ref{fig:lvi} has two children with complex weights and associated with different values of the latent $Z$. A sum-of-squares (Figure \ref{fig:square_before_sum}) gives a monotone PC, where the parameters necessarily become non-negative. A square-of-sums (Figure \ref{fig:sum_before_square}) leads to a squared PC, with four children each corresponding to the product of any two children from the original circuit.}
    \label{fig:circuit1}
\end{figure*}

In the latent variable interpretation (LVI) of probabilistic circuits \citep{peharz2016latent}, for every sum node, one assigns a categorical latent variable, where each state of the latent variable is associated with one of the inputs to the sum node; we show an example in Figure \ref{fig:lvi}.  In this interpretation, when performing inference in the probabilistic circuit, we marginalize over all of the latent variables beforehand.

In the case when the circuit is structured decomposable, one can assign latent variables with variable scopes (sets) appearing in the circuit, such that two sum nodes $n, n'$ with the same scope $\scope(n) = \scope(n')$ are associated with the same latent $Z_{\scope(n)}$. Then, writing $\bm{Z}$ for the set of all latents, the function represented by the PC can be expressed as:
\begin{equation}
    f_{\circuit}(\vars) = \sum_{\bm{z}} f_{\circuit}(\vars, \bm{z})
\end{equation}
where $f_{\circuit}(\vars, \bm{z})$ is a \emph{product of input functions} for any value of $\bm{z}$.\footnote{Concretely, $f_{\circuit}(\vars, \bm{z})$ is a function obtained by traversing the circuit top-down, selecting one child of every sum node according to the value $\bm{z}$, and all children of a prouct node; cf. also the notion of induced subcircuits/trees \citep{chan2006robustness,zhao2016unified}.} In other words, the distribution represented by the PC is fully factorized conditioned on a latent value $\bm{z}$.

However, interpreting these latent variables becomes tricky when we consider probability distributions defined by squaring circuits. The key question is, does one marginalize out the latent variables before or after squaring? We show both options in Figures \ref{fig:square_before_sum} and \ref{fig:sum_before_square}. In Figure \ref{fig:square_before_sum}, we square before marginalizing $Z$. In this case, each sum node weight is multiplied by its conjugate, and we are left with a sum node with non-negative real parameters. On the other hand, if we marginalize before squaring, we have a sum node with four children and complex parameters. 
Interestingly, the former case is very similar to directly constructing a monotone PC, while the latter is more like an explicit squaring without latent variables. This suggests that we can switch between monotone and squared PCs simply \emph{by deciding whether to sum the latent variables inside or outside the square}. 
Using this perspective, we propose the following model, which explicitly introduces both type of latent variables into the circuit, producing a deep \emph{sum-of-square-of-sums}:

\begin{definition}[Inception PC\footnote{The nomenclature is inspired by the deep(er) layering of summation and squaring in Inception PCs.}]
    An Inception PC (\modelshort) $\circuit_{\text{\modelname}}$ is a smooth and structured-decomposable probabilistic circuit over observed variables $\vars \cup \onelatents \cup \twolatents$. The probability distribution of an Inception PC is defined by:
    \begin{equation} \label{eqn:inception_math}
    \probfunc_{\text{\modelname}}(\vars) = 
\frac{\sum_{\onelatentsval} \left|\sum_{\twolatentsval} \pcfunc_{\circuit_{\text{\modelname}}}(\vars, \onelatentsval, \twolatentsval) \right|^2 }{\sum_{\varsval} \sum_{\onelatentsval} \left|\sum_{\twolatentsval} \pcfunc_{\circuit_{\text{\modelname}}}(\varsval, \onelatentsval, \twolatentsval) \right|^2 }
\end{equation}
\end{definition}

As $\bm{U}$-latents are outside the square, and $\bm{W}$-latents are inside the square, we will refer to them as 1-norm and 2-norm latents respectively. The next Theorem shows that, given an Inception PC, we can efficiently ``materialize'' it into a tractable PC over just the observed variables $\vars$ representing the Inception PC's distribution $\probfunc_{\text{\modelname}}(\vars)$:%

\begin{theorem}[Tractability of \modelname{}PC] \label{thm:tract_model}
    Given an Inception PC $\circuit_{\text{\modelname}}$, it is possible to compute a smooth and structured decomposable circuit $\circuit_{\text{mat}}$ such that $\pcfunc_{\circuit_{\text{mat}}}(\vars) = \sum_{\onelatentsval} \left|\sum_{\twolatentsval} \pcfunc_{\circuit_{\text{\modelname}}}(\vars, \onelatentsval, \twolatentsval) \right|^2 $ of size and in time $O(|\circuit|^2)$.
\end{theorem}
\begin{proof}
     The most systematic way to see this is to take advantage of the compositional inference framework of \citet{vergari2021compositional}.
    We can marginalize out $\twolatents$ from $\circuit_{\text{\modelname}}$ to obtain a PC $\circuit_1$ such that $\pcfunc_{\circuit_1}(\vars, \onelatents) = \sum_{\twolatentsval} \pcfunc_{\circuit_{\text{\modelname}}}(\vars, \onelatents, \twolatentsval)$, retaining smoothness and structured decomposability. Then the computation of the square is possible by Proposition \ref{prop:conjugation}, returning a smooth and structured-decomposable circuit $\circuit_2$ such that $\pcfunc_{\circuit_2}(\vars, \onelatents) = |\sum_{\twolatentsval} \pcfunc_{\circuit_{\text{\modelname}}}(\vars, \onelatents, \twolatentsval)|^2$. Finally, we can marginalize out $\bm{U}$ from this circuit to obtain the structured decomposable and smooth circuit $\circuit_{\text{mat}}$.
\end{proof}

In Figure \ref{fig:aug_inc}, we show a (fragment of) an example Inception PC, and its materialization via Theorem \ref{thm:tract_model}. In the rest of the paper, we will refer to this as the \emph{materialized} Inception PC. The point is that we can efficiently perform inference on the distribution of an Inception PC through standard PC inference procedures on the materialized \modelshort. Though this is not strictly necessary, we will assume (in accordance with the LVI for monotone and squared PCs) that $\bm{W}, \bm{U}$ are categoricals, and the input nodes with scope over these variables are indicators (e.g. $\llbracket W=1 \rrbracket$). %

This provides an elegant resolution to the tension between monotone and squared (real/complex) PCs. To retrieve a monotone PC, we need only set $\twolatents = \emptyset$; then there is no summation inside the square, and $\circuit_{\text{mat}}$ has the same structure as $\circuit_{\text{\modelname}}$ but with the parameters and input functions squared (and so non-negative real).\footnote{Interestingly enough, in this case we can relax the conditions of Theorem \ref{thm:tract_model} to require decomposability rather than structured decomposability, assuming that children of sum nodes correspond to different indicators and thus $\circuit_{\text{\modelname}}$ is deterministic. Multiplying a \emph{deterministic} circuit with itself (or its conjugate) is tractable in linear time \citep{vergari2021compositional}. %
} To retrieve a squared PC, we simply set $\onelatents = \emptyset$; then there is no summation outside the square. However, by using both types of latents, we obtain a generalized PC model that is strictly more expressive efficient than either individually\footnote{The special cases can also be \emph{learned} in the general case by making $f_{\circuit_{\text{\modelname{}}}}$ constant w.r.t. $\onelatents$ or $\twolatents$.}.

\begin{figure*}[t]
\begin{subfigure}{0.48\linewidth}
\centering
    \begin{tikzpicture}[semithick]
\tikzstyle{sum1}=[draw,circle,inner sep=1.2pt, cyan, ultra thick]
        \tikzstyle{prod}=[draw,circle,inner sep=1.2pt,orange, ultra thick]
        \tikzstyle{input}=[draw, circle,inner sep=2pt,black,ultra thick]
        \begin{scope}
            \node[sum1, minimum size=0.7cm] (sum1) at (0,-0.5) {$\bm{+}$};
            \node[prod] (prod1) at (-3,-1.5) {$\bm{\times_1}$}; 
            \node[prod] (prod2) at (-0.8,-1.5) {$\bm{\times_2}$}; 

    \node[prod] (prod3) at (.8,-1.5) {$\bm{\times_3}$}; 
                \node[prod] (prod4) at( 3,-1.5) {$\bm{\times_4}$};

            \node[input] (W0) at (-2.5, -2.8) {$\circ$};%
            \node[input] (W1) at (-1.3, -2.8) {$\circ$};%

            \node[] (W0label) at (-2.5, -3.2) {\scriptsize $\llbracket W = 0 \rrbracket$};
            \node[] (W1label) at (-1.3, -3.2) {\scriptsize $\llbracket W = 1 \rrbracket$};

            \node[input] (U0) at (-2.5+3.8, -2.8) {$\circ$};%
            \node[input] (U1) at (-1.3+3.8, -2.8) {$\circ$};%

            \node[] (U0label) at (-2.5+3.8, -3.2) {\scriptsize $\llbracket U = 0 \rrbracket$};
            \node[] (U1label) at (-1.3+3.8, -3.2) {\scriptsize $\llbracket U = 1 \rrbracket$};

            \draw[->] (sum1) -- (prod1) node[midway, above left] {\scriptsize $-2 + i$};
            \draw[->] (sum1) -- (prod2) node[midway, left] {\scriptsize $1$};
            \draw[->] (sum1) -- (prod3) node[midway, right] {\scriptsize $3i$};
            \draw[->] (sum1) -- (prod4) node[midway, above right] {\scriptsize $1 + 3i$};

            \draw[->] (prod1) -- (W0);
            \draw[->] (prod1) -- (U0);
            \draw[->] (prod2) -- (W1);
            \draw[->] (prod2) -- (U0);
            \draw[->] (prod3) -- (W0);
            \draw[->] (prod3) -- (U1);
            \draw[->] (prod4) -- (W1);
            \draw[->] (prod4) -- (U1);
        \end{scope}
    \end{tikzpicture}
    \caption{Inception PC $\circuit_{\text{\modelname{}}}$}
    \label{fig:incpc_example}
\end{subfigure}
\begin{subfigure}{0.48\linewidth}
\centering
    \begin{tikzpicture}[semithick]
\tikzstyle{sum1}=[draw,circle,inner sep=1.2pt, cyan, ultra thick]
        \tikzstyle{prod}=[draw,circle,inner sep=0.4pt,orange, ultra thick]
        \tikzstyle{input}=[draw, circle,inner sep=2pt,black,ultra thick]
        \begin{scope}
            \node[sum1, minimum size= 0.7cm] (sum1) at (0,-0.5) {$\bm{+}$};
            \node[prod] (prod11) at (-4.2,-2) {$\bm{\times_{11}}$}; 
            \node[prod] (prod21) at (-3,-2) {$\bm{\times_{12}}$}; 
            \node[prod] (prod12) at (-1.8,-2) {$\bm{\times_{21}}$}; 
            \node[prod] (prod22) at (-0.6,-2) {$\bm{\times_{22}}$};

            \draw[->] (sum1) -- (prod11) ;
            \draw[->] (sum1) -- (prod21) ;
            \draw[->] (sum1) -- (prod12) ;
            \draw[->] (sum1) -- (prod22);

            \node[] (prod11label) at (-4.2,-2.6) {\scriptsize $5$};
            \node[] (prod12label) at (-3.0,-2.6) {\scriptsize $2 + i$};
            \node[] (prod21label) at (-1.8,-2.6) {\scriptsize $2-i$};
            \node[] (prod22label) at (-0.6,-2.6) {\scriptsize $1$};

            \node[prod] (prod33) at (0.6,-2) {$\bm{\times_{33}}$}; 
            \node[prod] (prod43) at (1.8,-2) {$\bm{\times_{34}}$}; 
            \node[prod] (prod34) at (3.0,-2) {$\bm{\times_{43}}$}; 
            \node[prod] (prod44) at (4.2,-2) {$\bm{\times_{44}}$};

            \draw[->] (sum1) -- (prod33) ;
            \draw[->] (sum1) -- (prod34) ;
            \draw[->] (sum1) -- (prod43) ;
            \draw[->] (sum1) -- (prod44);

            \node[] (prod33label) at (0.6,-2.6) {\scriptsize $9$};
            \node[] (prod34label) at (1.8,-2.6) {\scriptsize $9-3i$};
            \node[] (prod43label) at (3.0,-2.6) {\scriptsize $9+3i$};
            \node[] (prod44label) at (4.2,-2.6) {\scriptsize $10$};

        \end{scope}
    \end{tikzpicture}
    \caption{Materialized Inception PC $\circuit_{\text{mat}}$}
\end{subfigure}
    \caption{Diagrams showing an Inception PC, and the corresponding materialized \modelshort{}. Each product node is labelled with an index, such that e.g. $\times_{34}$ is the product of the product nodes $\times_3, \times_4$. For clarity, the children of product nodes have been omitted (except the latent indicators), and edge weights for the materialized Inception PC are displayed below the corresponding child.}
    \label{fig:aug_inc}
\end{figure*}

\begin{restatable}[Expressive Efficiency of Inception PCs]{corollary}{corExpressive}
    Inception PCs are strictly more expressive efficient than both (structured-decomposable) monotone and squared PCs.
\end{restatable}

\section{Tensorized Inception PCs} \label{sec:tensorized}

Thus far, we have described a general formulation of Inception PCs that simply requires smoothness and structured decomposability of the circuit. In practice, we will want to design specific circuit architectures for learning. The purpose of this section is to (i) propose such an architecture suitable for learning and inference on GPUs; and (ii) present an alternative view of Inception PCs as performing hierarchical tensor contractions.

We follow recent trends in probabilistic circuit learning \citep{peharz2020einsum,mari2023unifying} and consider tensorized architectures for $\circuit_{\text{\modelname{}}}$, where sum and product nodes are grouped into regions by scope. In particular, for each region (scope), we construct $N_1 \times N_2$ sum and product nodes in the Inception PC, where the sum nodes are connected in a dense fashion to the product nodes via a weight tensor $\weight_{ii'kk'} \in \mathbb{C}^{N_1^2 \times N_2^2}$ with $1 \leq k, k' \leq N_1$, $1 \leq i, i' \leq N_2$; and the product nodes are connected to sum nodes in subsequent regions via a Hadamard product \citep{loconte2024relationship}. Each region is associated with a categorical latent $W$ and $U$; and each product node corresponds explicitly to a value of $W$ and $U$, having two indicators $\llbracket W = i' \rrbracket$ and $\llbracket U = k' \rrbracket$ as children where $1 \leq k' \leq N_1$, $1 \leq i' \leq N_2$. For example, the product nodes in the Inception PC in Figure \ref{fig:incpc_example} correspond to a tensorized Inception PC with $N_1 = 2, N_2 = 2$. 

Once materialized using Theorem \ref{thm:tract_model}, we obtain a materialized \modelshort{} $\circuit_{\text{mat}}$ with $N_1 \times N_2^2$ sum and product nodes; the increase in size occurs from the ``expansion'' from two-norm latents $\bm{W}$ as seen in Figure \ref{fig:sum_before_square}. For a given region, let us write $n^{(S)}_{ijk}$ for the sum nodes for $1 \leq k \leq N_1, 1 \leq i, j \leq N_2$, and $n^{(P)}_{ijk}$ for the product nodes for $1 \leq k \leq N_1, 1 \leq i, j \leq N_2$ in $\circuit_{\text{mat}}$. The product nodes are still connected to their $M$ child sum regions $\node^{(S_1)}, ..., \node^{(S_M)}$ via a Hadamard product:
\begin{equation}
    \pcfunc_{\node^{(P)}_{ijk}} = \prod_{m=1}^{M} \pcfunc_{\node^{(S_m)}_{ijk}}
\end{equation}
However, the sum-to-product connection is no longer dense. In particular, writing we have the following relationship between the sum and product node tensors:
\begin{equation} \label{eqn:tensor_contract}
    \pcfunc_{\node^{(S)}_{ijk}} = \sum_{i'j'k'} \weight_{ii'kk'} \overline{\weight_{jj'kk'}} \pcfunc_{\node^{(P)}_{i'j'k'}}
\end{equation}

We illustrate the structure of this sum region in Figure \ref{fig:tensorPC}. On a high level, the connections between the groups of nodes in Figure \ref{fig:tensorPC} can be viewed as a standard, monotone PC; in particular, the value of each sum group is simply an weighted sum of its children. However, each weighted edge is between 2D groups of ``squared'' nodes, rather than between scalar nodes.

\begin{figure}[t]
    \centering
    \tikzstyle{sumx}=[inner sep=1pt]
\tikzstyle{prodx}=[inner sep=1pt]
\tikzstyle{input}=[inner sep=1pt,minimum size=12pt]

\begin{tikzpicture}[semithick]
        \begin{scope}
            \node[sumx] (s1) at (0,0) {$\bm{+}$};
            \node[sumx] (s2) at (0.5,-0) {$\bm{+}$};
            \node[sumx] (s3) at (0, 0.5) {$\bm{+}$};
            \node[sumx] (s4) at (0.5,0.5) {$\bm{+}$};

            \draw[rounded corners, color=cyan, very thick] (-0.4,-0.4) rectangle (0.9,0.9) {};

            \node[sumx] (s5) at (0+3,0) {$\bm{+}$};
            \node[sumx] (s6) at (0.5+3,-0) {$\bm{+}$};
            \node[sumx] (s7) at (0+3, 0.5) {$\bm{+}$};
            \node[sumx] (s8) at (0.5+3,0.5) {$\bm{+}$};

            \draw[rounded corners, color=cyan, very thick] (-0.4+3,-0.4) rectangle (0.9+3,0.9) {};

            \node[sumx] (s9) at (0+6,0) {$\bm{+}$};
            \node[sumx] (s10) at (0.5+6,-0) {$\bm{+}$};
            \node[sumx] (s11) at (0+6, 0.5) {$\bm{+}$};
            \node[sumx] (s12) at (0.5+6,0.5) {$\bm{+}$};

            \draw[rounded corners, color=cyan, very thick] (-0.4+6,-0.4) rectangle (0.9+6,0.9) {};

            \node[prodx] (p1) at (0,0-3) {$\bm{\times}$};
            \node[prodx] (p2) at (0.5,-0-3) {$\bm{\times}$};
            \node[prodx] (p3) at (0, 0.5-3) {$\bm{\times}$};
            \node[prodx] (p4) at (0.5,0.5-3) {$\bm{\times}$};

            \draw[rounded corners, color=orange, very thick] (-0.4,-0.4-3) rectangle (0.9,0.9-3) {};

            \node[prodx] (p5) at (0+3,0-3) {$\bm{\times}$};
            \node[prodx] (p6) at (0.5+3,-0-3) {$\bm{\times}$};
            \node[prodx] (p7) at (0+3, 0.5-3) {$\bm{\times}$};
            \node[prodx] (p8) at (0.5+3,0.5-3) {$\bm{\times}$};

            \draw[rounded corners, color=orange, very thick] (-0.4+3,-0.4-3) rectangle (0.9+3,0.9-3) {};

            \node[prodx] (p9) at (0+6,0-3) {$\bm{\times}$};
            \node[prodx] (p10) at (0.5+6,-0-3) {$\bm{\times}$};
            \node[prodx] (p11) at (0+6, 0.5-3) {$\bm{\times}$};
            \node[prodx] (p12) at (0.5+6,0.5-3) {$\bm{\times}$};

            \draw[rounded corners, color=orange, very thick] (-0.4+6,-0.4-3) rectangle (0.9+6,0.9-3) {};

            \draw[thick] (0.25,-0.4) -- (0.25, -2.1);
            \draw[thick] (0.25,-0.4) -- (3.25, -2.1);
            \draw[thick] (0.25,-0.4) -- (6.25, -2.1);
            \draw[thick] (3.25,-0.4) -- (0.25, -2.1);
            \draw[thick] (3.25,-0.4) -- (3.25, -2.1);
            \draw[thick] (3.25,-0.4) -- (6.25, -2.1);
            \draw[thick] (6.25,-0.4) -- (0.25, -2.1);
            \draw[thick] (6.25,-0.4) -- (3.25, -2.1);
            \draw[thick] (6.25,-0.4) -- (6.25, -2.1);

            \node[] (i1) at (0.2, 1.2) {\scriptsize $i=1, 2$};
            \node[ rotate=90] (j) at (-0.7, 0.3) {\scriptsize $j = 1, 2$};
            \node[] (k1) at (0.2, -3.7) {\scriptsize $k = 1$};
            \node[] (k2) at (3.2, -3.7) {\scriptsize $k = 2$};
            \node[] (k3) at (6.2, -3.7) {\scriptsize $k = 3$};
          \end{scope}
    \end{tikzpicture}
    \caption{Illustration of tensorized Inception PC sum region, where $N_1 = 3$, $N_2 = 2$.}
    \label{fig:tensorPC}
\end{figure}

The complexity of a forward pass can be deduced from the tensor contraction in Equation \ref{eqn:tensor_contract}. In Table \ref{tbl:complexity}, we compare the parameter counts (i.e., memory cost) and time complexity of a forward pass for each region, between monotone, squared and Inception PCs; where $B$ is the size of the data batch. It can be seen that Inception PCs have the same number of parameters/complexity as monotone and squared PCs when setting $N_2 = 1$ and $N_1 = 1$ respectively, except with the complexity differing for squared PCs: this is because one can perform the squaring only once-per-batch in this special case \cite{LoconteICLR24}.

As with previous works \citep{peharz2020einsum}, we organize the circuit into layers (sets of regions that can be computed simultaneously) to further take advantage of GPU parallelization. For training, we use gradient descent on the negative log-likelihood of the training set. In the case of using complex parameters, we use Wirtinger derivatives \citep{kreutz2009complex} in order to optimize the complex weights and input functions. To achieve numerical stability, we use a variant of the log-sum-exp trick for complex numbers, which we describe in Appendix \ref{apx:logsumexp_complex}.

\begin{table}[t]
\begin{tabular}{@{}llll@{}}
\toprule
           & \textbf{Monotone} & \textbf{Squared} & \textbf{Inception} \\ \midrule
Complexity & $BN_1^2$           & $BN_2^2 + N_2^3$ & $BN_1^2N_2^3$      \\
Parameter Count & $N_1^2$           & $N_2^2$ & $N_1^2N_2^2$      \\\bottomrule
\end{tabular}
\caption{Complexity of batched forward pass and parameter count for circuit variants, per region.}
\label{tbl:complexity}
\end{table}

\begin{table*}[t]
\centering
\scalebox{0.88}{
    
\begin{tabular}{@{}cccccccc|cc@{}}
\toprule
\multicolumn{1}{c}{\textbf{Dataset}} & \multicolumn{9}{c}{\textbf{Model}}                                                                         \\ \cmidrule(l){2-10} 
                                     & Monotone PC & \multicolumn{3}{c}{Squared PC} & \multicolumn{3}{c}{\modelname{} PC}  & RAT-SPN & VAE                          \\ 
                                     &            & Non-negative    & Real  & Complex      & Non-negative & Real & Complex \\ \midrule
nltcs                                &   6.04        &       6.02       &  6.09  &     6.03   &               \textbf{6.00}         &    6.01   &               6.02  & 6.01 &   5.99          \\
msnbc                      &       6.25     &       6.23        & 6.10  &    6.07       &       6.05              &     6.06    &              \textbf{6.04}     & 6.04  &     6.09       \\
kdd-2k                     &    2.13        &     \textbf{2.10}       &  2.23  &   2.12        &   2.12             &     2.13       &         2.12      & 2.13 &  2.12         \\
plants                        &     13.70       &        13.38        & 15.06  &    13.13       &              12.81       &    13.06    &   \textbf{ 12.76 }                   & 13.44  &   12.34     \\ 
jester                        &    52.77       &        52.56        & 54.51   &   52.97       &                     \textbf{52.51}        &   52.60  & 52.75                  & 52.97 &    51.54       \\ 
audio                      &     40.27        &        40.08   &  41.49 &       40.01      &                    \textbf{39.88}    &  39.91  &    40.05                &  39.96  &   38.67        \\ 
netflix                       &     57.09        &        56.85   &    57.68 &       56.70       &                     \textbf{56.52}    &  56.57    &   56.74      & 56.85 &   54.73               \\ 
accidents                       &     29.57        &        27.93    & 28.15   &       27.05       &                     26.70    & 27.30   &    \textbf{26.61} & 35.49  & 29.11  \\
retail                       &     10.99        &        10.82   & 11.00   &       \textbf{10.95}       &                    11.00      & \textbf{10.95}  &   \textbf{10.95}  & 10.91&  10.83 \\ 
pumsb-star                       &     27.98   & 24.95     &        25.69   &  23.98  &       23.69       &                     24.85          &   \textbf{23.03} & 32.53 & 25.16    \\
dna                       &     80.21       &        79.95     & 80.15 &       80.17       &             79.85         &  80.11 & \textbf{79.77} & 97.23 &   94.56 \\
kosarek                      &    10.77        &      \textbf{10.54}       & 12.03 &   10.59          &        10.69        &  10.83& 10.60  & 10.89 &  10.64 \\
msweb                       &     10.44        &        \textbf{9.92}      & 10.58 &       10.17       &             10.84        &  10.34 & 10.10 & 10.12 &  9.727  \\
book                       &     33.70        &        \textbf{33.32}     & 37.02 &       33.95      &             33.51         &  34.18 & 33.67   & 34.68 & 33.19 \\
   eachmovie                   &     52.83       &        51.28    & 62.03 &       52.33     &            \textbf{50.76}        &  51.22 & 51.41 & 53.63 &  47.43  \\
web-kb                      &     155.34        &        151.84     & 162.03 &       155.00      &             \textbf{151.74}        &  153.32 & 153.98   & 157.53 & 146.9 \\
reuters-52                     &     95.22        &        92.63     & 96.25 &       93.90      &             93.17      &  \textbf{89.67} & 93.80 & 87.37 &  81.33  \\
20ng                     &     155.05       &       \textbf{ 152.98 }    & 164.19 &       154.79     &             154.15        &  155.47 & 155.18 &152.06 & 146.9   \\
bbc                      &     253.98       &        \textbf{250.88}    & 259.04 &      255.13        &           251.28      &  253.28 & 253.37 & 252.14&   240.94 \\
ad                      &     16.93     &    15.54 &    16.32    & 15.35 &       16.02      &            15.81      &  \textbf{15.32} & 48.47 & 18.81  \\
\bottomrule
\end{tabular}
    }
    \caption{Test negative log-likelihoods on 20 binary datasets. Lower is better.}
    \label{tbl:debd}
\end{table*}

\section{Related Work}

The space of probabilistic models can be usefully characterized through the lens of \emph{tractability} and \emph{expressive efficiency} \citep{ProbCirc20}. Tractability forms a spectrum, extending from intractable models such as diffusion models \citep{ho2020denoising} and variational autoencoders \citep{Kingma2014}, to models admitting tractable likelihoods such as normalizing flows \citep{papamakarios2021normalizing}, to models admitting tractable marginals and more complex queries \citep{vergari2021compositional,WangNeurIPS24} such as probabilistic circuits. Probabilistic circuits, however, are generally less expressive efficient in practice than less tractable models due to the structure that is imposed in the copmutation graph. Thus much effort has been expended on improving the expressive efficiency of probabilistic circuits whilst maintaining tractability \citep{sidheekh2024building}.

Our work builds upon a long line of work in the circuit literature examining the effect of relaxing the monotonicity condition in circuits. Theoretically, it is known that allowing negative parameters in arithmetic circuits can result in exponential gains in succinctness \citep{valiant1979negation}, though this not true of all circuit subclasses \citep{decolnet2021succinctness}.
Practically, recent works have aimed to exploit negative parameters in probabilistic modeling \citep{zhang2020relationship,zhang2021probabilistic,sladek2023encoding,LoconteICLR24}. \citet{yu2023characteristic} design circuits to represent the characteristic function of a distribution, which can take complex values (in particular, using complex leaves).
Concurrently with our work, \citet{loconte2024sumsquarescircuits} proposed to employ a sum of squared circuits to overcome similar limitations to those we observe in this paper. 

There also exist a range of other probabilistic models which employ negative parameterizations and/or squaring. Tensor networks, such as the popular matrix-product states (MPS) \citep{perez2007matrix}, compactly encode functions through sparse tensor contractions. They are often used to model quantum states, whereby the probability of observations is given by squaring via the Born rule \citep{dirac1981principles}; recently, tensor networks have also been used for probabilistic modeling \citep{cheng2019tree,glasser2019expressive}. Positive semidefinite models \citep{rudi2021psd} in the kernel methods literature utilize a shallow sum of squares to define unnormalized distributions. \citet{tsuchida2023squared,tsuchida2024exact} recently introduced squared neural families, which employ the squared 2-norm of a neural network in the density function and strictly generalize exponential familiies. We discuss technical connections with these models in Appendix \ref{apx:other_models}.%

\section{Experiments} \label{sec:experiments}

In our experiments, we aim to answer the following research questions: %
(1) do Inception PCs improve upon the expressivity and modeling capabilities of monotone and squared PCs?; (2) what is the tradeoff between the number of one-norm latents and two-norm latents in terms of modeling performance and computational cost?; (3) how do complex parametererizations of Inception PCs compare to non-negative and real parameterizations?

We use hidden Chow-Liu trees (HCLT) in all experiments as the PC vtree \cite{liu2021tractable}, as it satisfies the required structured-decomposability property, and has been shown to provide state-of-the-art likelihoods for PCs. Further experimental details can be found in Appendix \ref{apx:experiment_details}.

\subsection{Binary Datasets}

We begin by evaluating on the 20 binary datasets, which has been extensively used as a density estimation benchmark \citep{rooshenas2014learning,peharz2020ratspn}. We aim to investigate across this range of datasets how (1) non-negative, real, and complex parameterizations differ; and (2) whether Inception PCs can effectively make use of both types of latents, by comparing to its monotone and squared PC counterparts. In particular, we set $(N_1, N_2) = (8, 1)$ for monotone PCs, $(N_1, N_2) = (1, 8)$ for squared PCs, and $(N_1, N_2) = (8, 8)$ for Inception PCs. We train each PC on negative log-likelihood for 100 epochs, using the Adam optimizer \citep{kingma2015adam} with learning rate $0.01$. We average over 5 runs for each configuration.

The results are shown in Table \ref{tbl:debd}. It can be seen that Inception PCs achieve the best likelihoods on 14 of the 20 datasets, confirming that they provide more modeling capacity compared to squaring alone. They also perform very favorably compared to existing models, including (tractable) RAT-SPNs \citep{peharz2020ratspn} and (intractable) importance weighed autoencoders \citep{burda2016importance}. Squared PCs, even with non-negative parameters, outperform their monotone counterparts that have the same number of parameters, as also observed in Figure 3 of \citet{LoconteICLR24}. Very interestingly, despite the fact that complex parameterizations of edge weights strictly generalizes real parameterizations and non-negative parameterizations, there is no clear trend in performance, with the non-negative parameterization often achieving comparable or better likelihoods than real or complex parameterizations. We observed that this is (at least partially) due to the complex and negative parameterizations being more prone to overfitting (cf. training LLs in Appendix \ref{apx:experiment_train}).

\subsection{Scaling to Large Image Datasets}

In this section, we conduct experiments on large-scale image datasets, 
in particular, downscaled versions of ImageNet to $32 \times 32$ and $64 \times 64$ \cite{deng2009imagenet}. Following recent work on PC modeling for these datasets \cite{LiuICLR23,liu2023understanding,liu2024scaling}, we transform the data from RGB using the lossy YCoCg transform. Note that likelihoods on on YCoCg transformed data are thus not comparable to likelihoods on the original RGB dataset. Additionally, to improve training efficiency, we train the circuits on 16 $\times$ 16 patches of the original image\footnote{The entire image is then modelled using the same PC for each of the patches; the performance could potentially be improved further by modelling correlation using a  PC over patch-level PCs \citep{liu2023understanding}.}, such that the PC models $16 \times 16 \times 3 = 768$ variables each with $256$ categories. We train the models on negative log-likelihood, using the Adam optimizer with learning rate $0.01$. We evaluate models using test-set bits-per-dimension (bpd).

We begin by examining the tradeoff between the two types of latents. To this end, in Figure \ref{fig:imagenet32_configs} we plot the test bpd for a range of configurations of $(N_1, N_2)$ on ImageNet32. These configurations were chosen based on a search in the range $N_1, N_2 \in \{1, 2, 4, 8, 16, 32, 64\}$ according to a maximum budget of $2^{18} = 262144$ floating-point operations per second (FLOPS) per region (c.f. Table \ref{tbl:complexity}). It can be seen that the optimal configuration ($N_1=32,N_2=4$ with bpd $4.19$) given these constraints uses a combination of 1-norm latents and 2-norm latents, as opposed to a pure monotone ($N_2=1$) or pure squared circuit ($N_1=1$).

In Table \ref{tbl:image_results}, we %
summarize results from the ImageNet32 and ImageNet64 datasets, where MPC refers to a monotonic PC $(N_1, N_2) = (64, 1)$, SQC a squared PC $(N_1, N_2) = (1, 64)$, and IncPC a tensorized Inception PC with the optimal configuration of $(N_1, N_2)$ under the computation time constraint. 
 For comparison, we also show results for monotone HCLT PCs trained using expectation-maximization (EM), and using latent variable distillation (LVD) in which guidance for the PC latent space is provided by distilling information from existing deep generative models \citep{LiuICLR23}. The results show that it is possible to effectively train tractable PC models from initialization with Adam with bpds competitive with the state-of-the-art.

\begin{figure}[t]
    \centering
    \includegraphics[width=\linewidth]{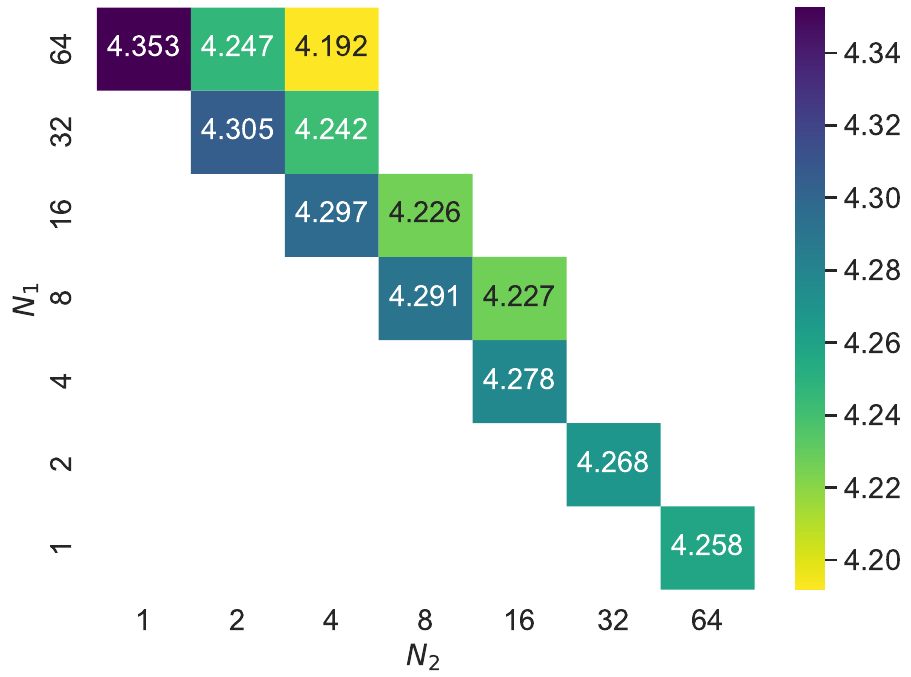}
    \caption{Test bpd for a range of configurations of $(N_1, N_2)$ on ImageNet32 (lower is better); configurations are limited to $2^{18} = 262K$ FLOPS per region. The Inception PC with $N_1=32, N_2=4$ achieves the best performance.}
    \label{fig:imagenet32_configs}
\end{figure}

\begin{table}[t]
\centering
\begin{tabular}{@{}llll|ll@{}}
\toprule
           & MPC & SPC   & IncPC   & EM   & LVD  \\ \midrule
ImageNet32 &  4.35   & 4.26 & \textbf{4.19} & 4.82 & 4.38 \\
ImageNet64 &   4.12 & 3.98  & \textbf{3.90}  & 4.67 & 4.12 \\ \bottomrule
\end{tabular}
\caption{Test bpd on large-scale image datasets (lower is better). All PCs use the HCLT vtree. }
\label{tbl:image_results}
\end{table}

\section{Conclusion}

To conclude, we have shown that two important classes of tractable probabilistic models, namely monotone and squared real structured-decomposable PCs are incomparable in terms of expressive efficiency in general. Thus, we propose a new class of probabilistic circuits, Inception PCs, based on \emph{deep sums-of-squares-of-sums} that generalizes these approaches, and can employ complex parameters. We further show empirically that Inception PCs can offer better performance on a range of tabular and image datasets. %
Promising avenues to investigate in future work would be improving the optimization of complex parameters in \modelname{} PCs; as well as reducing the computational cost of training by designing more efficient architectures.

\section*{Acknowledgements}
BW and GVdB gratefully acknowledge support from the National Artificial Intelligence Research Resource Pilot under project NAIRR240143. This work was funded in part by the DARPA ANSR
program under award FA8750-23-2-0004, the DARPA
PTG Program under award HR00112220005, and NSF
grant \#IIS-1943641.

\bibliography{references}

\begin{thebibliography}{57}
\providecommand{\natexlab}[1]{#1}

\bibitem[{Ahmed, Chang, and Van~den Broeck(2023)}]{AhmedNeurIPS23}
Ahmed, K.; Chang, K.-W.; and Van~den Broeck, G. 2023.
\newblock A Pseudo-Semantic Loss for Deep Generative Models with Logical
  Constraints.
\newblock In \emph{Advances in Neural Information Processing Systems 36
  (NeurIPS)}.

\bibitem[{Ahmed et~al.(2022)Ahmed, Teso, Chang, Van~den Broeck, and
  Vergari}]{AhmedNeurIPS22}
Ahmed, K.; Teso, S.; Chang, K.-W.; Van~den Broeck, G.; and Vergari, A. 2022.
\newblock Semantic Probabilistic Layers for Neuro-Symbolic Learning.
\newblock In \emph{Advances in Neural Information Processing Systems 35
  (NeurIPS)}.

\bibitem[{Broadrick, Zhang, and Van~den Broeck(2024)}]{broadrick2024polynomial}
Broadrick, O.; Zhang, H.; and Van~den Broeck, G. 2024.
\newblock Polynomial Semantics of Tractable Probabilistic Circuits.
\newblock In \emph{Conference on Uncertainty in Artificial Intelligence}. PMLR.

\bibitem[{Burda, Grosse, and Salakhutdinov(2016)}]{burda2016importance}
Burda, Y.; Grosse, R.; and Salakhutdinov, R. 2016.
\newblock Importance Weighted Autoencoders.
\newblock In \emph{International Conference on Learning Representations
  (ICLR)}.

\bibitem[{Chan and Darwiche(2006)}]{chan2006robustness}
Chan, H.; and Darwiche, A. 2006.
\newblock On the robustness of most probable explanations.
\newblock In \emph{Proceedings of the Twenty-Second Conference on Uncertainty
  in Artificial Intelligence}, 63--71.

\bibitem[{Cheng et~al.(2019)Cheng, Wang, Xiang, and Zhang}]{cheng2019tree}
Cheng, S.; Wang, L.; Xiang, T.; and Zhang, P. 2019.
\newblock Tree tensor networks for generative modeling.
\newblock \emph{Physical Review B}, 99(15): 155131.

\bibitem[{Choi, Kisa, and Darwiche(2013)}]{choi2013compiling}
Choi, A.; Kisa, D.; and Darwiche, A. 2013.
\newblock Compiling probabilistic graphical models using sentential decision
  diagrams.
\newblock In \emph{Symbolic and Quantitative Approaches to Reasoning with
  Uncertainty: 12th European Conference, ECSQARU 2013, Utrecht, The
  Netherlands, July 8-10, 2013. Proceedings 12}, 121--132. Springer.

\bibitem[{Choi, Vergari, and Van~den Broeck(2020)}]{ProbCirc20}
Choi, Y.; Vergari, A.; and Van~den Broeck, G. 2020.
\newblock Probabilistic Circuits: A Unifying Framework for Tractable
  Probabilistic Models.

\bibitem[{Darwiche(2003)}]{darwiche2003differential}
Darwiche, A. 2003.
\newblock A differential approach to inference in Bayesian networks.
\newblock \emph{Journal of the ACM (JACM)}, 50(3): 280--305.

\bibitem[{de~Colnet and Mengel(2021)}]{decolnet2021succinctness}
de~Colnet, A.; and Mengel, S. 2021.
\newblock A compilation of succinctness results for arithmetic circuits.
\newblock In \emph{18th International Conference on Principles of Knowledge
  Representation and Reasoning (KR)}.

\bibitem[{Deng et~al.(2009)Deng, Dong, Socher, Li, Li, and
  Fei-Fei}]{deng2009imagenet}
Deng, J.; Dong, W.; Socher, R.; Li, L.-J.; Li, K.; and Fei-Fei, L. 2009.
\newblock Imagenet: A large-scale hierarchical image database.
\newblock In \emph{2009 IEEE conference on computer vision and pattern
  recognition}, 248--255. Ieee.

\bibitem[{Dirac(1981)}]{dirac1981principles}
Dirac, P. A.~M. 1981.
\newblock \emph{The principles of quantum mechanics}.
\newblock 27. Oxford university press.

\bibitem[{Fawzi et~al.(2015)Fawzi, Gouveia, Parrilo, Robinson, and
  Thomas}]{fawzi2015positive}
Fawzi, H.; Gouveia, J.; Parrilo, P.~A.; Robinson, R.~Z.; and Thomas, R.~R.
  2015.
\newblock Positive semidefinite rank.
\newblock \emph{Mathematical Programming}, 153: 133--177.

\bibitem[{Glasser et~al.(2019)Glasser, Sweke, Pancotti, Eisert, and
  Cirac}]{glasser2019expressive}
Glasser, I.; Sweke, R.; Pancotti, N.; Eisert, J.; and Cirac, I. 2019.
\newblock Expressive power of tensor-network factorizations for probabilistic
  modeling.
\newblock \emph{Advances in neural information processing systems}, 32.

\bibitem[{Han et~al.(2018)Han, Wang, Fan, Wang, and
  Zhang}]{han2018unsupervised}
Han, Z.-Y.; Wang, J.; Fan, H.; Wang, L.; and Zhang, P. 2018.
\newblock Unsupervised generative modeling using matrix product states.
\newblock \emph{Physical Review X}, 8(3): 031012.

\bibitem[{Ho, Jain, and Abbeel(2020)}]{ho2020denoising}
Ho, J.; Jain, A.; and Abbeel, P. 2020.
\newblock Denoising diffusion probabilistic models.
\newblock \emph{Advances in neural information processing systems}, 33:
  6840--6851.

\bibitem[{Kingma and Ba(2015)}]{kingma2015adam}
Kingma, D.~P.; and Ba, J. 2015.
\newblock Adam: {A} Method for Stochastic Optimization.
\newblock In Bengio, Y.; and LeCun, Y., eds., \emph{3rd International
  Conference on Learning Representations, {ICLR} 2015, San Diego, CA, USA, May
  7-9, 2015, Conference Track Proceedings}.

\bibitem[{Kingma and Welling(2014)}]{Kingma2014}
Kingma, D.~P.; and Welling, M. 2014.
\newblock {Auto-Encoding Variational Bayes}.
\newblock In \emph{2nd International Conference on Learning Representations,
  {ICLR} 2014, Banff, AB, Canada, April 14-16, 2014, Conference Track
  Proceedings}.

\bibitem[{Kreutz-Delgado(2009)}]{kreutz2009complex}
Kreutz-Delgado, K. 2009.
\newblock The complex gradient operator and the CR-calculus.
\newblock \emph{arXiv preprint arXiv:0906.4835}.

\bibitem[{Liu, Ahmed, and Van~den Broeck(2024)}]{liu2024scaling}
Liu, A.; Ahmed, K.; and Van~den Broeck, G. 2024.
\newblock Scaling Tractable Probabilistic Circuits: A Systems Perspective.
\newblock In \emph{Forty-first International Conference on Machine Learning}.

\bibitem[{Liu, Niepert, and Van~den Broeck(2024)}]{LiuICLR24}
Liu, A.; Niepert, M.; and Van~den Broeck, G. 2024.
\newblock Image Inpainting via Tractable Steering of Diffusion Models.
\newblock In \emph{Proceedings of the Twelfth International Conference on
  Learning Representations (ICLR)}.

\bibitem[{Liu and Van~den Broeck(2021)}]{liu2021tractable}
Liu, A.; and Van~den Broeck, G. 2021.
\newblock Tractable regularization of probabilistic circuits.
\newblock \emph{Advances in Neural Information Processing Systems}, 34:
  3558--3570.

\bibitem[{Liu, Zhang, and Van~den Broeck(2023)}]{LiuICLR23}
Liu, A.; Zhang, H.; and Van~den Broeck, G. 2023.
\newblock Scaling Up Probabilistic Circuits by Latent Variable Distillation.
\newblock In \emph{Proceedings of the International Conference on Learning
  Representations (ICLR)}.

\bibitem[{Liu et~al.(2023)Liu, Liu, Van~den Broeck, and
  Liang}]{liu2023understanding}
Liu, X.; Liu, A.; Van~den Broeck, G.; and Liang, Y. 2023.
\newblock Understanding the distillation process from deep generative models to
  tractable probabilistic circuits.
\newblock In \emph{International Conference on Machine Learning}, 21825--21838.
  PMLR.

\bibitem[{Loconte et~al.(2024{\natexlab{a}})Loconte, Mari, Gala, Peharz,
  de~Campos, Quaeghebeur, Vessio, and Vergari}]{loconte2024relationship}
Loconte, L.; Mari, A.; Gala, G.; Peharz, R.; de~Campos, C.; Quaeghebeur, E.;
  Vessio, G.; and Vergari, A. 2024{\natexlab{a}}.
\newblock What is the Relationship between Tensor Factorizations and Circuits
  (and How Can We Exploit it)?
\newblock \emph{arXiv preprint arXiv:2409.07953}.

\bibitem[{Loconte, Mengel, and Vergari(2025)}]{loconte2024sumsquarescircuits}
Loconte, L.; Mengel, S.; and Vergari, A. 2025.
\newblock Sum of Squares Circuits.
\newblock In \emph{The 39th Annual AAAI Conference on Artificial Intelligence}.

\bibitem[{Loconte et~al.(2024{\natexlab{b}})Loconte, Sladek, Mengel, Trapp,
  Solin, Gillis, and Vergari}]{LoconteICLR24}
Loconte, L.; Sladek, A.~M.; Mengel, S.; Trapp, M.; Solin, A.; Gillis, N.; and
  Vergari, A. 2024{\natexlab{b}}.
\newblock Subtractive Mixture Models via Squaring: Representation and Learning.
\newblock In \emph{Proceedings of the Twelfth International Conference on
  Learning Representations (ICLR)}.

\bibitem[{Mari, Vessio, and Vergari(2023)}]{mari2023unifying}
Mari, A.; Vessio, G.; and Vergari, A. 2023.
\newblock Unifying and Understanding Overparameterized Circuit Representations
  via Low-Rank Tensor Decompositions.
\newblock In \emph{The 6th Workshop on Tractable Probabilistic Modeling}.

\bibitem[{Marteau-Ferey, Bach, and Rudi(2020)}]{marteau2020non}
Marteau-Ferey, U.; Bach, F.; and Rudi, A. 2020.
\newblock Non-parametric models for non-negative functions.
\newblock \emph{Advances in neural information processing systems}, 33:
  12816--12826.

\bibitem[{Martens and Medabalimi(2014)}]{martens2014expressive}
Martens, J.; and Medabalimi, V. 2014.
\newblock On the expressive efficiency of sum product networks.
\newblock \emph{arXiv preprint arXiv:1411.7717}.

\bibitem[{Papamakarios et~al.(2021)Papamakarios, Nalisnick, Rezende, Mohamed,
  and Lakshminarayanan}]{papamakarios2021normalizing}
Papamakarios, G.; Nalisnick, E.; Rezende, D.~J.; Mohamed, S.; and
  Lakshminarayanan, B. 2021.
\newblock Normalizing flows for probabilistic modeling and inference.
\newblock \emph{Journal of Machine Learning Research}, 22(57): 1--64.

\bibitem[{Peharz et~al.(2016)Peharz, Gens, Pernkopf, and
  Domingos}]{peharz2016latent}
Peharz, R.; Gens, R.; Pernkopf, F.; and Domingos, P. 2016.
\newblock On the latent variable interpretation in sum-product networks.
\newblock \emph{IEEE transactions on pattern analysis and machine
  intelligence}, 39(10): 2030--2044.

\bibitem[{Peharz et~al.(2020{\natexlab{a}})Peharz, Lang, Vergari, Stelzner,
  Molina, Trapp, Van~den Broeck, Kersting, and Ghahramani}]{peharz2020einsum}
Peharz, R.; Lang, S.; Vergari, A.; Stelzner, K.; Molina, A.; Trapp, M.; Van~den
  Broeck, G.; Kersting, K.; and Ghahramani, Z. 2020{\natexlab{a}}.
\newblock Einsum networks: Fast and scalable learning of tractable
  probabilistic circuits.
\newblock In \emph{International Conference on Machine Learning}, 7563--7574.
  PMLR.

\bibitem[{Peharz et~al.(2020{\natexlab{b}})Peharz, Vergari, Stelzner, Molina,
  Shao, Trapp, Kersting, and Ghahramani}]{peharz2020ratspn}
Peharz, R.; Vergari, A.; Stelzner, K.; Molina, A.; Shao, X.; Trapp, M.;
  Kersting, K.; and Ghahramani, Z. 2020{\natexlab{b}}.
\newblock Random sum-product networks: A simple and effective approach to
  probabilistic deep learning.
\newblock In \emph{Uncertainty in Artificial Intelligence}, 334--344. PMLR.

\bibitem[{Perez-Garcia et~al.(2007)Perez-Garcia, Verstraete, Wolf, and
  Cirac}]{perez2007matrix}
Perez-Garcia, D.; Verstraete, F.; Wolf, M.; and Cirac, J. 2007.
\newblock Matrix product state representations.
\newblock \emph{Quantum Information \& Computation}, 7(5): 401--430.

\bibitem[{Pipatsrisawat and Darwiche(2008)}]{pipatsrisawat08compilation}
Pipatsrisawat, K.; and Darwiche, A. 2008.
\newblock New Compilation Languages Based on Structured Decomposability.
\newblock In \emph{Proceedings of the Twenty-Third AAAI Conference on
  Artificial Intelligence (AAAI)}, 517--522.

\bibitem[{Poon and Domingos(2011)}]{poon2011sum}
Poon, H.; and Domingos, P. 2011.
\newblock Sum-product networks: A new deep architecture.
\newblock In \emph{2011 IEEE International Conference on Computer Vision
  Workshops (ICCV Workshops)}, 689--690. IEEE.

\bibitem[{Rahman, Kothalkar, and Gogate(2014)}]{rahman2014cutset}
Rahman, T.; Kothalkar, P.; and Gogate, V. 2014.
\newblock Cutset networks: A simple, tractable, and scalable approach for
  improving the accuracy of chow-liu trees.
\newblock In \emph{Machine Learning and Knowledge Discovery in Databases:
  European Conference, ECML PKDD 2014, Nancy, France, September 15-19, 2014.
  Proceedings, Part II 14}, 630--645. Springer.

\bibitem[{Rooshenas and Lowd(2014)}]{rooshenas2014learning}
Rooshenas, A.; and Lowd, D. 2014.
\newblock Learning sum-product networks with direct and indirect variable
  interactions.
\newblock In \emph{International Conference on Machine Learning}, 710--718.
  PMLR.

\bibitem[{Rosser and Schoenfeld(1962)}]{rosser1962approximate}
Rosser, J.~B.; and Schoenfeld, L. 1962.
\newblock Approximate formulas for some functions of prime numbers.
\newblock \emph{Illinois Journal of Mathematics}, 6(1): 64--94.

\bibitem[{Rudi and Ciliberto(2021)}]{rudi2021psd}
Rudi, A.; and Ciliberto, C. 2021.
\newblock PSD representations for effective probability models.
\newblock \emph{Advances in Neural Information Processing Systems}, 34:
  19411--19422.

\bibitem[{Shao et~al.(2022)Shao, Molina, Vergari, Stelzner, Peharz, Liebig, and
  Kersting}]{shao2022conditional}
Shao, X.; Molina, A.; Vergari, A.; Stelzner, K.; Peharz, R.; Liebig, T.; and
  Kersting, K. 2022.
\newblock Conditional sum-product networks: Modular probabilistic circuits via
  gate functions.
\newblock \emph{International Journal of Approximate Reasoning}, 140: 298--313.

\bibitem[{Sidheekh and Natarajan(2024)}]{sidheekh2024building}
Sidheekh, S.; and Natarajan, S. 2024.
\newblock Building expressive and tractable probabilistic generative models: a
  review.
\newblock In \emph{Proceedings of the Thirty-Third International Joint
  Conference on Artificial Intelligence}, 8234--8243.

\bibitem[{Sladek, Trapp, and Solin(2023)}]{sladek2023encoding}
Sladek, A.~M.; Trapp, M.; and Solin, A. 2023.
\newblock Encoding Negative Dependencies in Probabilistic Circuits.
\newblock In \emph{The 6th Workshop on Tractable Probabilistic Modeling}.

\bibitem[{Tsuchida, Ong, and Sejdinovic(2023)}]{tsuchida2023squared}
Tsuchida, R.; Ong, C.~S.; and Sejdinovic, D. 2023.
\newblock Squared Neural Families: A New Class of Tractable Density Models.
\newblock \emph{Advances in Neural Information Processing Systems}, 36.

\bibitem[{Tsuchida, Ong, and Sejdinovic(2024)}]{tsuchida2024exact}
Tsuchida, R.; Ong, C.~S.; and Sejdinovic, D. 2024.
\newblock Exact, Fast and Expressive Poisson Point Processes via Squared Neural
  Families.
\newblock In \emph{Proceedings of the AAAI Conference on Artificial
  Intelligence}, volume~38, 20559--20566.

\bibitem[{Valiant(1979)}]{valiant1979negation}
Valiant, L.~G. 1979.
\newblock Negation can be exponentially powerful.
\newblock In \emph{Proceedings of the eleventh annual ACM symposium on theory
  of computing}, 189--196.

\bibitem[{Vergari et~al.(2021)Vergari, Choi, Liu, Teso, and Van~den
  Broeck}]{vergari2021compositional}
Vergari, A.; Choi, Y.; Liu, A.; Teso, S.; and Van~den Broeck, G. 2021.
\newblock A compositional atlas of tractable circuit operations for
  probabilistic inference.
\newblock \emph{Advances in Neural Information Processing Systems}, 34:
  13189--13201.

\bibitem[{Wang and Kwiatkowska(2023)}]{wang2023compositional}
Wang, B.; and Kwiatkowska, M. 2023.
\newblock Compositional Probabilistic and Causal Inference using Tractable
  Circuit Models.
\newblock In Ruiz, F.; Dy, J.; and van~de Meent, J.-W., eds., \emph{Proceedings
  of The 26th International Conference on Artificial Intelligence and
  Statistics}, volume 206 of \emph{Proceedings of Machine Learning Research},
  9488--9498. PMLR.

\bibitem[{Wang et~al.(2024)Wang, Maua, Van~den Broeck, and
  Choi}]{WangNeurIPS24}
Wang, B.; Maua, D.; Van~den Broeck, G.; and Choi, Y. 2024.
\newblock A Compositional Atlas for Algebraic Circuits.
\newblock In \emph{Advances in Neural Information Processing Systems 37
  (NeurIPS)}.

\bibitem[{Wang, Wicker, and Kwiatkowska(2022)}]{wang2022trust}
Wang, B.; Wicker, M.~R.; and Kwiatkowska, M. 2022.
\newblock Tractable Uncertainty for Structure Learning.
\newblock In Chaudhuri, K.; Jegelka, S.; Song, L.; Szepesvari, C.; Niu, G.; and
  Sabato, S., eds., \emph{Proceedings of the 39th International Conference on
  Machine Learning}, volume 162 of \emph{Proceedings of Machine Learning
  Research}, 23131--23150. PMLR.

\bibitem[{Yu, Trapp, and Kersting(2023)}]{yu2023characteristic}
Yu, Z.; Trapp, M.; and Kersting, K. 2023.
\newblock Characteristic Circuits.
\newblock \emph{Advances in Neural Information Processing Systems}, 36.

\bibitem[{Zhang et~al.(2023)Zhang, Dang, Peng, and Van~den
  Broeck}]{ZhangICML23}
Zhang, H.; Dang, M.; Peng, N.; and Van~den Broeck, G. 2023.
\newblock Tractable Control for Autoregressive Language Generation.
\newblock In \emph{Proceedings of the 40th International Conference on Machine
  Learning (ICML)}.

\bibitem[{Zhang, Holtzen, and Van~den Broeck(2020)}]{zhang2020relationship}
Zhang, H.; Holtzen, S.; and Van~den Broeck, G. 2020.
\newblock On the relationship between probabilistic circuits and determinantal
  point processes.
\newblock In \emph{Conference on Uncertainty in Artificial Intelligence},
  1188--1197. PMLR.

\bibitem[{Zhang, Juba, and Van~den Broeck(2021)}]{zhang2021probabilistic}
Zhang, H.; Juba, B.; and Van~den Broeck, G. 2021.
\newblock Probabilistic generating circuits.
\newblock In \emph{International Conference on Machine Learning}, 12447--12457.
  PMLR.

\bibitem[{Zhang et~al.(2025)Zhang, Wang, Arenas, and Van~den
  Broeck}]{ZhangAISTATS25}
Zhang, H.; Wang, B.; Arenas, M.; and Van~den Broeck, G. 2025.
\newblock Restructuring Tractable Probabilistic Circuits.
\newblock In \emph{Proceedings of the 28th International Conference on
  Artificial Intelligence and Statistics (AISTATS)}.

\bibitem[{Zhao, Poupart, and Gordon(2016)}]{zhao2016unified}
Zhao, H.; Poupart, P.; and Gordon, G.~J. 2016.
\newblock A unified approach for learning the parameters of sum-product
  networks.
\newblock \emph{Advances in neural information processing systems}, 29.

\end{thebibliography}

\maketitle

\newpage
\appendix
\onecolumn
\section{Relationship with Other Models} \label{apx:other_models}

In this section, we explain the relationship of Inception PCs with other tractable model classes (i.e. probabilistic models admitting efficient computation of normalizing constants) employing similar principles. %

\subsection{Probabilistic Circuits} \label{apx:other_models_pc}
\citet{loconte2024sumsquarescircuits} concurrently introduced \emph{sum of compatible squares} (SOCS), which models a distribution as a shallow sum of squared circuits:
\begin{equation} \label{eqn:SOCS}
    p_{\text{SOCS}}(\vars) \propto \sum_{i=0}^{n-1} |f_{i}(\vars)| ^ 2
\end{equation}
where $f_1, ..., f_n$ are probabilistic circuits that are \emph{compatible}, i.e. share the same scope decomposition. Using the latent variable interpretation $f_{i}(\bm{V}) = \sum_{\bm{w}} f_{i}(\bm{V}, \bm{w})$, we can interpret this as a special case of an InceptionPC where there is only one 1-norm latent variable $U_{\vars}$ taking values in $0, \ldots, n-1$, and with augmented PC function given by \begin{equation}
    f_{\circuit_{\text{aug}}}(\vars, u_{\vars}, \bm{w}) = f_{u_{\vars}}(\vars, \bm{w})
\end{equation}

In addition, the authors also proposed taking the product of compatible monotone and squared circuits ($\mu$SOCS), i.e.
\begin{equation} \label{eqn:muSOCS}
    p_{\mu\text{SOCS}}(\vars) \propto f_1(\bm{V}) |f_{2}(\bm{V})|^2
\end{equation}
Once again, we can employ the latent variable interpretation $f_1(\bm{V}) = \sum_{u} f_1(\bm{V}, \bm{u})$ and $f_2(\vars) = \sum_{\bm{w}} f_2(\vars, \bm{w})$, which can be represented using the following augmented PC function:
\begin{align}
    f_{\circuit_{\text{aug}}}(\vars, \bm{u}, \bm{w}) &= \sqrt{f_{1}(\vars, \bm{u})} f_{2}(\vars, \bm{w})
\end{align}
That is $\mu$SOCS circuits correspond to a particular factorization of the augmented PC over $\bm{u}, \bm{w}$, such that the 1-norm and 2-norm latents are enforced to be ``independent''. In other words, both SOCS and $\mu$SOCS can be viewed as instances of Inception PCs.

A key advantage compared to general Inception PCs, however, is that they enable the computation of unnormalized likelihoods without explicitly performing the squaring operation. In particular, the log-RHS of equations \ref{eqn:SOCS} and \ref{eqn:muSOCS} can be computed by computing the $f_i$ (or $f_1, f_2$) on a given sample $\bm{V}=\bm{v}$ and then computing the expression (similar trick to squared PCs). This is not possible for general Inception PCs, because of the dependence between $\bm{U}$ and $\bm{W}$. Thus, to train e.g. $\mu$SOCS one needs to perform the squaring explicitly only once per batch/parameter update to obtain the normalizing constant. This leads to a training time complexity of $O\bigl(N_1^2N_2^3 + B(N_1^2 + N_2^2)\bigr)$ per region, which can be significantly cheaper than the $O(BN_1^2N_2^3)$ for large batch sizes. For inference queries such as probability or marginal computation, the complexity is the same unless the queries can be batched effectively.
The improved time complexity of $\mu$SOCS, however, may come at the cost of expressivity and modeling performance. An interesting open question is thus to determine whether there is a separation between $\mu$SOCS and the general class of all Inception PCs. 

\subsection{Tensor Networks}

In the tensor network literature, it has been proposed both to use non-negative tensor factorizations and the square of complex tensor factorizations \citep{han2018unsupervised} as parameterizations of a tensor with non-negative entries and thus to represent multidimensional discrete probability distributions. While tensor networks are intractable in general, the widely used matrix-product states (MPS) / tensor-train (TT) and tree tensor networks (TTN) admit tractable inference, and can be interpreted as structured-decomposable circuits; in particular, right-linear vtrees for MPS and general vtrees for TTN \cite{LoconteICLR24,loconte2024relationship}.

Particularly relevant in our context are the locally purified states (LPS) introduced by \citet{glasser2019expressive}, which are a generalization of squared MPS. In the original notation, this factorizes a tensor $T_{X_1, \ldots, X_n}$ over variables (indices) $X_1, \ldots, X_n$ as follows:
\begin{equation}
    T_{X_1, ..., X_n} = \sum_{\{\alpha_i, \alpha'_i =1\}}^{r} \sum_{\{\beta_i = 1\}}^{\mu} A_{1, X_1}^{\beta_1, \alpha_1} \overline{A_{1, X_1}^{\beta_1, \alpha_1'}} A_{2, X_2}^{\beta_2, \alpha_1, \alpha_2} \overline{A_{2, X_2}^{\beta_2, \alpha_1', \alpha_2'}} \ldots A_{n, X_n}^{\beta_n, \alpha_n} \overline{A_{n, X_n}^{\beta_n, \alpha_n'}} 
\end{equation}
where $A_1, ..., A_n$ are tensors. 
In contrast, when employing a right-linear vtree, our Inception PCs can be written using the following tensor factorization (matching notation to the extent possible):
\begin{equation}
    T_{X_1, ..., X_n} = \sum_{\{\alpha_i, \alpha'_i =1\}}^{r} \sum_{\{\beta_i = 1\}}^{\mu} A_{1}^{\beta_1, \alpha_1} \overline{A_{1}^{\beta_1, \alpha_1'}} E_{1, X_1}^{\beta_1, \alpha_1} \overline{E_{1, X_1}^{\beta_1, \alpha_1'}} A_{2}^{\beta_1, \beta_2, \alpha_1, \alpha_2} \overline{A_{2}^{\beta_1, \beta_2, \alpha_1', \alpha_2'}} E_{2, X_2}^{\beta_2, \alpha_2} \overline{E_{2, X_2}^{\beta_2, \alpha'_2}} \ldots
\end{equation}
Here, the $A$ tensors correspond to sum node weight matrices, and $E$ tensors correspond to the input categorical functions. Intuitively, one can view the $\alpha_i$ indices as corresponding to the 2-norm latents $\bm{W}$ and $\beta_i$ to the 1-norm latents $\bm{U}$. Thus Inception PCs share similarities with LPS in extending beyond pure monotone or squared tensor factorizations. The key differences with LPS are that (i) $X_i$ only depends on $\alpha_i, \alpha_i'$ (not $\alpha_{i-1}$) and (ii) $\beta_{i-1}, \beta_{i}$ appear together in the transition tensor. 

Thus, with this interpretation, LPS and Inception PCs can be understood essentially as different patterns for constructing a non-negative tensor/probability distribution using complex parameters. LPS could also be in theory generalized to arbitrary vtrees (i.e. tree tensor network) with some work. The practical choice of which to use depends on two main factors. Firstly, it can be checked that the complexity of evaluting the LPS is $O(BN_1N_2^3d)$ per region, where $\mu = N_1$, $r = N_2$, and $d$ is the dimension of the indices (i.e. the number of categories of each $X$ variable). This is as compared to the $O(BN_1^2N_2^3)$ complexity for Inception PCs. Secondly, one should consider the performance scaling with the $N_1$ and $N_2$ parameters. The experiments of \citet{glasser2019expressive} showed little to no improvement of LPS over Born machines in terms of log-likelihood, which correspond to setting $N_1 > 1$ or $N_1 = 1$ respectively. In contrast, our experiments show significant improvements in log-likelihood when increasing $N_1$. This suggests that the Inception PC pattern is more effective with scaling w.r.t. the 1-norm latents.

\subsection{Squared Neural Families and PSD models}

Positive semi-definite (PSD) kernel models \citep{marteau2020non,rudi2021psd} specify probability distributions as:
\begin{equation}
    p(\bm{x}) \propto 
    \bm{\kappa}(\bm{x})^T \bm{A} \bm{\kappa}(\bm{x}) 
\end{equation}
where $\bm{\kappa}$ is some feature map (which is assumed to be finite-dimensional) and $\bm{A}$ a positive semi-definite matrix. Since this model is \emph{shallow}, \citet{sladek2023encoding} proposed to embed these models within a deep probabilistic circuit architecture by introducing \emph{PSD} nodes, where the feature map is given by a vector of circuit nodes. That is, the function given by the PSD node is defined as $\bm{f}(\bm{x})^T \bm{A} \bm{f}(\bm{x})$ where $\bm{f}(\bm{x}) = [f_1(\bm{x}), ... f_n(\bm{x})]$ is a vector of node functions. They considered in particular settings where the PSD nodes are either located at the leaves of the PC, or at the root of the PC. In both such cases one can encode this model efficiently as a $\mu$SOCS and thus an Inception PC \citep{loconte2024sumsquarescircuits}.

\citet{tsuchida2023squared} recently introduced squared neural families (SNEFY), a class of probabilistic models that specify unnormalized probability densities by squaring a neural network's output, i.e.,
\begin{equation}
    P(d\bm{x}; \bm{V}, \bm{\Theta}) = \frac{\mu(d\bm{x})}{Z(\bm{V}, \bm{\Theta})} || f(\bm{t}(\bm{x}); \bm{V}, \bm{\Theta}) ||^2
\end{equation}
where $\bm{X}$ are the variables and $\bm{V}$ and $\bm{\Theta}$ are the neural network parameters, $\bm{t}$ is a sufficient statistic, and $Z(\bm{V}, \bm{\Theta})$ is a normalizing constant. Under certain conditions on $f$, $\bm{t}$ and base measure $\mu$, including but not limited to exponential families, it is possible to efficiently compute the normalizing constant of this distribution efficiently. It was shown that some of these SNEFYs can be encoded as $\mu$SOCS \citep{loconte2024sumsquarescircuits}, and thus Inception PCs by Section \ref{apx:other_models_pc}. Another point of interest is that tractable SNEFYs can be used as expressive conditional probability distributions (i.e. $p(\bm{x}|\bm{y})$) by making the parameters a function of the condition $\bm{y}$ in a particular way. One can also use Inception PCs as a conditional probability model, either by (i) explicitly introducing evidence on a circuit respecting the joint distribution $p(\bm{x}, \bm{y})$ -- in which case it is not too hard to see that the result remains an Inception PC -- or (ii) by similarly making the InceptionPC parameters a (neural network) function of the condition \citep{shao2022conditional}.

We also refer readers to the excellent technical comparison in \citep{loconte2024sumsquarescircuits} for further details on the relationship between these model types and circuits. 

\section{Proofs}

\propConjugation*
\begin{proof}
    We show the first part inductively from leaves to the root. By assumption, we can compute the complex conjugate of the input functions. Thus we need to show that we can compute the conjugate of the sums and products efficiently, assuming that we can compute the conjugates of their inputs.

    Suppose that we have a sum $\node$; then we have that:
    $\overline{\pcfunc_{\node}} = \overline{\sum_{\node_i \in \inputs(\node)} \weight_{\node, \node_i} \pcfunc_{\node_i}} = \sum_{\node_i \in \inputs(\node)} \overline{\weight_{\node, \node_i}} \;\overline{\pcfunc_{\node_i}}$. Thus we can simply conjugate the weights and take the conjugated input nodes.

    Suppose that we are given a product $\node$; then we have that:
    $\overline{\pcfunc_{\node}} = \overline{\prod_{\node_i \in \inputs(\node)} \pcfunc_{\node_i}} = \prod_{\node_i \in \inputs(\node)} \overline{\pcfunc_{\node_i}}$. Thus we can take the conjugated input nodes.

    This procedure is clearly linear time and keeps exactly the same structure as the original circuit (thus smoothness and decomposability). If the input circuit is structured decomposable, then we can multiply $\pcfunc_{\circuit}$ and $\pcfunc_{\overline{\circuit}}$ as they are compatible \citep{vergari2021compositional}, producing a smooth and structured decomposable circuit as output.
\end{proof}

\thmMonoSuccinct*
\begin{proof}
    Given a set of $d$ variables $\bm{V}$, we consider the function:
    \begin{equation}
        p(\bm{V}) = n(\bm{V}) + 1
    \end{equation}
    where we write $n(\bm{V})$ for the  non-negative integers given by the binary representation. 
    
    \paragraph{Existence of Compact Str.Dec.Monotone Circuit} This function can be easily represented as a linear-size monotone structured-decomposable PC as follows:
    \begin{align*}
        p(\bm{V}) = \pcfunc_{\circuit}(\bm{V}) = \sum_{i=0}^{d-1} 2^i \mathds{1}_{V_i = 1} + 1
    \end{align*}
    which can also be easily smoothed if desired. 

    \paragraph{Lower Bound Strategy} It remains to show the lower bound on the size of the negative structured-decomposable PC $\circuit'$. Firstly, we have the following Lemma:
    \begin{lemma} \citep{martens2014expressive}
        Let $F$ be a function over variables $\bm{V}$ computed by a structured-decomposable and smooth circuit $\circuit$. Then there exists a partition of the variables $(\bm{X}, \bm{Y})$ with $\frac{1}{3} |\bm{V}| \leq |\bm{X}|, |\bm{Y}| \leq \frac{2}{3} |\bm{V}| $ and $N \leq |\circuit|^2$ such that:
        \begin{equation} \label{eq:factor}
            F(\bm{X}, \bm{Y}) = \sum_{i=1}^{N} G_i(\bm{X}) \times H_i(\bm{Y})
        \end{equation}
        for some functions $G_i, H_i$. 
    \end{lemma}
    Variations of this result have appeared multiple times in literature, e.g. Theorem 38 in \cite{martens2014expressive} and Theorem 2 in \citep{decolnet2021succinctness}. Unlike prior results, we do not place the restriction that $F$ is non-negative; this is a simple modification of the existing proof as this was only required to ensure the non-negativity of $G_i, H_i$, which we do not require.
    
    To show a lower bound on $|\circuit'|$, we can thus show a lower bound on $N$ for all functions $F(\bm{X}, \bm{Y}) = \pm \sqrt{n(\bm{V}) + 1}$. To do this, we use another Lemma (Lemma 13 in \citet{decolnet2021succinctness}):
    \begin{definition}
        Given a function $F$ over variables $\bm{X}, \bm{Y}$, we define the value matrix $M_{F(\bm{X}, \bm{Y})} \in \mathbb{R}^{2^{|\bm{X}|} \times 2^{|\bm{Y}|}}$ by:
        \begin{equation}
            M_{n(\bm{X}), n(\bm{Y})} := F(\bm{X}, \bm{Y})
        \end{equation}
    \end{definition}
    \begin{lemma} \citep{decolnet2021succinctness}
        Suppose Equation \ref{eq:factor} holds. Then $rank(M_{F(\bm{X}, \bm{Y})}) \leq N$.
    \end{lemma}
    Thus, it suffices to lower bound $rank(M_{F(\bm{X}, \bm{Y})})$ over all partitions $\bm{X}, \bm{Y}$ such that $\frac{1}{3} |\bm{V}| \leq |\bm{X}|, |\bm{Y}| \leq \frac{2}{3} |\bm{V}|$.
    
    \paragraph{Lower Bound} Given such a partition $\bm{X}, \bm{Y}$, assume w.l.o.g. $|\bm{X}| \leq |\bm{Y}|$. Consider any function $F(\bm{V})$ such that $F(\bm{V}) = \pm \sqrt{n(\bm{V}) + 1} 
$.%
    
    Each variable $X \in \bm{X}$ corresponds to some variable in $\bm{V}$. We write $idx(X)$ to denote the \emph{index} of the variable $X$ corresponds to; for example, if $X$ is $V_4$, then $idx(X) = 4$. 
    Then we have the following:
    \begin{equation}
        F(\bm{X}, \bm{Y}) = \pm \sqrt{\sum_{i=0}^{|\bm{X}|-1} 2^{idx(X_i)} X_i + \sum_{i=0}^{d-|\bm{X}|-1} 2^{idx(Y_i)} Y_i + 1}
    \end{equation} 
    We write $\iota(\bm{X}) := \sum_{i=0}^{|\bm{X}|-1} 2^{idx(i)} X_i$ and $\iota(\bm{Y}) := \sum_{i=0}^{d-|\bm{X}|-1} 2^{idx(i)} Y_i$ such that $F(\bm{X}, \bm{Y}) = \pm \sqrt{\iota(\bm{X}) + \iota(\bm{Y}) + 1}$. Note that $\iota$ is injective as the $idx(X_i)$ are 
    distinct for each $i$ (sim. for $idx(Y_i)$). 
    
    Now we need the following Lemma: %

    \begin{lemma} \label{lem:rowcol}
        For any $\epsilon > 0$, and for sufficiently large $d$, there exists at least $M = 2^{(\frac{1}{4} - \epsilon) d}$ distinct instantiations of $\{\bm{x}_i\}_{0=1}^{M-1}$ and $M$ distinct instantiations of $\{\bm{y}_i\}_{i=0}^{M-1} $ of $\bm{Y}$ such that $p_i := \iota(\bm{x_i}) + \iota(\bm{y_i}) + 1$ are distinct primes, and $\iota(\bm{x_j}) + \iota(\bm{y_k}) + 1 \neq p_i$ for any $0 \leq i, j, k \leq M - 1$ except $i = j = k$. 
     \end{lemma}
     \begin{proof}
         We begin by lower bounding the number of \emph{prime pairs}; that is, the number of instantiations $(\bm{x}, \bm{y})$ of $\bm{X}, \bm{Y}$ such that $(F(\bm{x}, \bm{y}))^2 = \iota(\bm{x}) + \iota(\bm{y}) + 1$ is prime. Each prime $p$ less than or equal to $2^{d}$ will have exactly 1 prime pair. %
         The number of primes $\pi(m)$ less than or equal to any given integer $m \geq 17$ is lower bounded by $\frac{m}{\ln m}$ \cite{rosser1962approximate}. Thus, we have that the number of prime pairs is at least:
          \begin{equation}
             \frac{2^d}{d \ln 2}
         \end{equation}

         Given any instantiation $\bm{x}$ of $\bm{X}$, we call $\bm{y}$ a \emph{prime completion} of $\bm{x}$ if $(\bm{x}, \bm{y})$ is a prime pair. We now claim that there are at least $M$ instantiations of $\bm{X}$ such that each has at least $2M^2 + 1$ prime completions. Suppose for contradiction this was not the case. Then the total number of prime pairs is upper bounded by:
         \begin{align*}
             &(M - 1) \times 2^{d - |\bm{X}|} + (2^{|\bm{X}|} - M + 1) \times 2M^2 \\
            &< 2^{(\frac{1}{4} - \epsilon) d} \times 2^{d - |\bm{X}|} + 2^{|\bm{X}|} \times 2^{(\frac{1}{2} - 2\epsilon) d + 1} \\
            &= 2^{(\frac{5}{4} - \epsilon) d - |\bm{X}|} + 2^{(\frac{1}{2} - 2\epsilon) d + |\bm{X}| + 1} \\
            &\leq 2^{(\frac{11}{12} - \epsilon) d} + 2^{(1 - 2\epsilon) d + 1}
         \end{align*}
         The first line is an upper bound on the number of prime pairs in this case; $(M - 1)$ instantiations of $\bm{X}$ with any $\bm{y}$ being a potential prime completion ($2^{d - |\bm{X}|}$ total), and the rest having at most $2M^2$ prime completions. The second line follows by substituting $M$, the third by rearrangement, and the fourth using the fact that $\frac{1}{3}d \leq |\bm{X}| \leq \frac{1}{2}d$. But this upper bound is less than the lower bound above, for sufficiently large $d$. Thus, we have a contradiction. 

         Now, to finish the Lemma, we describe an algorithm for picking the $M$ instantiations $\{\bm{x}_i\}_{i=0}^{M-1}$, $\{\bm{y}_i\}_{i=0}^{M-1}$. From the claim above, we have $M$ instantiations $\{\bm{x}_i\}_{i=0}^{M-1}$ each with at least $2M^2 + 1$ prime completions. We iterate over $m = 0, ..., M-1$. Suppose that at iteration $m$, we have already chosen $\{\bm{y}_i\}_{i=0}^{m-1}$ such that $p_i := \iota(\bm{x_i}) + \iota(\bm{y_i}) + 1$ are distinct primes for $0 \leq i \leq m - 1$, and $\iota(\bm{x_j}) + \iota(\bm{y_k}) + 1 \neq p_i$ for any $0 \leq j \leq M - 1$ and $0 \leq i, k \leq m - 1$ except $i = j = k$. For $\bm{x_m}$, we aim to choose a prime completion $\bm{y_m}$ such that 
         \begin{align}
             &(i) \; \iota(\bm{x}_j) + \iota(\bm{y}_k) + 1 \neq \iota(\bm{x}_m) + \iota(\bm{y}_m) + 1 \\ 
             &(ii) \; \iota(\bm{x}_j) + \iota(\bm{y}_m) + 1 \neq p_k + 1
         \end{align}
         for any $0 \leq j \leq M - 1$ and any $0 \leq k \leq m$ except $j = k = m$. Thus, there are at most $2 * M * (m + 1) \leq 2M^2$ values that $\iota(\bm{y}_m)$ must not take; as we have $2M^2 + 1$ prime completions, we can always choose a $\bm{y}_m$ satsifying the conditions (i), (ii). Given conditions (i), (ii) together with the inductive hypothesis, we have that 
         $p_i := \iota(\bm{x_i}) + \iota(\bm{y_i}) + 1$ are distinct primes for $0 \leq i \leq (m - 1) + 1$, and $\iota(\bm{x_j}) + \iota(\bm{y_k}) + 1 \neq p_i$ for any $0 \leq j \leq M - 1$ and $0 \leq i, k \leq (m - 1) + 1$ except $i = j = k$.
    \end{proof}

    With this Lemma in hand, we can finish the argument as follows. By Lemma \ref{lem:rowcol}, we have $M$ distinct instantiations $\{\bm{x}_{i}\}_{i=0}^{M - 1}, \{\bm{y}_{i}\}_{i=0}^{M - 1}$ such that $p_i := \iota(\bm{x}_i) + \iota(\bm{y}_i)$ is prime for every $i$; suppose that these are ordered such that $p_0 < ... < p_{M-1}$. Now consider the submatrix $\bm{M'} \in \mathbb{R}^{M \times M}$ of $\bm{M}_{F(\bm{X}, \bm{Y})}$ obtained by taking the rows $(n(\bm{x}_i))_{i=0}^{M-1}$ and columns $(n(\bm{y}_i))_{i=0}^{M - 1}$ (in-order). The rank $rank(\bm{M}_{F(\bm{X}, \bm{Y})})$ is lower bounded by $rank(\bm{M}')$; thus, we seek to find $rank(\bm{M}')$.

     \begin{lemma} 
         $rank(\bm{M}') = M$
     \end{lemma}
     \begin{proof}
         This proof is a variation on Example 10 from  \cite{fawzi2015positive}, but differs from that matrix in that there is no guarantee that the entries are increasing. As such, we rely on a slightly different technique to prove this result.

         Recall that $F(\bm{X}, \bm{Y}) = \pm \sqrt{\iota(\bm{X}) + \iota(\bm{Y}) + 1}$. Thus, the matrix $\bm{M'}$ is given by:
         \begin{equation}
         \bm{M}'_{ij} = \pm \sqrt{\iota(\bm{x_i}) + \iota(\bm{y_j}) + 1}
         \end{equation}
         Now consider the submatrices $\bm{M'}^{(1)}, ..., \bm{M'}^{(M)}$ defined by $\bm{M'}^{(i)} := \bm{M'}_{0:i-1, 0:i-1}$ (i.e. the first $i$ rows and columns). We show by induction that $\bm{M'}^{(i)}$ has rank $i$. The base case $i = 1$ is clear. 
         
         For the inductive step, suppose that $\bm{M'}^{(i - 1)}$ has rank $i - 1$. Then consider $\bm{M'}^{(i)}$. Note that the square of the  bottom right entry $(\bm{M'}_{i-1,i-1})^2 = \iota(\bm{x}_{i-1}) + \iota(\bm{y}_{i-1}) + 1 = p_{i-1}$ is prime. We now claim that $(\bm{M'}_{jk})^2$ is not a positive integer multiple of $p_i$ for any $0 \leq j, k \leq i - 1$ except $j = k = i - 1$. 
         
         Firstly, by Lemma \ref{lem:rowcol} there is no $j, k$ such that $(\bm{M'}_{jk})^2 = \iota(\bm{x}_j) + \iota(\bm{y}_k) + 1 = p_i$ unless $j = k = i-1$, i.e. a multiple of $1$ is not possible. We further have that: \begin{align*}
             &\iota(\bm{x}_j) + \iota(\bm{y}_k) + 1 \\ &\leq \iota(\bm{x}_j) + \iota(\bm{y}_j) + \iota(\bm{y}_k) + \iota(\bm{x}_k) + 1 \\ &= p_j + p_k - 1 \\
             &< 2p_i - 1
         \end{align*}
         
          Thus $(\bm{M'}_{jk})^2$ cannot be a positive integer multiple of $p_i$. 

         Now, the determinant of the matrix $\bm{M'}^{(i)}$ takes the form $\alpha \bm{M'}_{i-1,i-1} + \beta$, where $\alpha$ is the determinant of $\bm{M'}^{(i - 1)}$, and $\beta$ is a multilinear function of the entries of $\bm{M}'^{(i)}$ \emph{excluding} $\bm{M'}_{i-1,i-1}$. Note that both $\alpha$ and $\beta$ are in the extension field $\mathbb{Q}[\sqrt{P_i}]$, where $P_i$ is defined to be the set of all primes that divide $(\bm{M'}_{jk})^2$ for some $0 \leq j, k \leq i - 1$ except $j = k = i - 1$.

         Now, $(\bm{M'}_{i-1,i-1})^2  = p_i$ is not in this set $P_i$, as we showed that no other $(\bm{M'}_{jk})^2$ can be a multiple of $p_i$. Thus $\bm{M'}_{i-1,i-1} = \pm \sqrt{p_i}$ is not in this extension field. 
         By the inductive assumption, $\alpha \neq 0$, and so $det(\bm{M'}^{(i)}) =\alpha \bm{M'}_{i-1, i-1} + \beta $ must be nonzero also. Thus $\bm{M}'^{(i)}$ has full rank, i.e. rank $i$.

    \end{proof}
     
    Putting it all together, we have shown that given any square root function $F(\bm{V}) = \pm \sqrt{n(\bm{V}) + 1}$ and any structured-decomposable and smooth circuit $\circuit'$ computing $F$, and any balanced partition $\bm{X}, \bm{Y}$ of $\bm{V}$, then for any $\epsilon > 0$ and sufficiently large $d$, we have a lower bound $2^{(\frac{1}{4} - \epsilon)d} < rank(\bm{M}') < rank(\bm{M}_{F(\bm{X}, \bm{Y})}) < |\circuit'|^2$.
\end{proof}

\corExpressive*
\begin{proof}
    As noted, Inception PCs can express both monotone and squared PCs as special cases. Inception PCs are strictly more expressive efficient than monotone circuits by Theorem \ref{thm:sq_succinct}, as they can express squared PCs; similarly, Inception PCs are strictly more expressive efficient than squared circuits by Theorem \ref{thm:monotone_succinct}.
\end{proof}

\section{Algorithms}

\begin{algorithm}[tbp]
    \KwInput{Smooth and (structured) decomposable probabilistic circuit with root node $n$}
    \If{$n$ \emph{is input node}}{\Return $\texttt{INPUT}(\overline{\pcfunc_{\node}})$}
    \ElseIf{$n$ \emph{is product node}}{\Return $\texttt{PROD}(\{\texttt{CONJ}(\node_i)\}_{\node_i \in \inputs(\node)})$}
    \ElseIf{$n$ \emph{is sum node}}{\Return $\texttt{SUM}(\{\texttt{CONJ}(\node_i)\}_{\node_i \in \inputs(\node)}, \{\overline{\weight_{\node, \node_i}}\}_{\node_i \in \inputs(\node)})$}
    \KwResult{Smooth and (structured) decomposable probabilistic circuit with root node $\overline{\node}$ such that $\pcfunc_{\overline{\node}} = \overline{\pcfunc_{\node}}$}
    \caption{Conjugation $\texttt{CONJ}(\node)$}
    \label{alg:conj}
\end{algorithm}

\begin{algorithm}[tbp]
    \KwInput{Smooth and (structured) decomposable probabilistic circuit with root node $n$; subset of variables $\bm{W} \subseteq \bm{V}$}
    \If{$n$ \emph{is input node}}{\Return $\texttt{INPUT}(\sum_{\bm{w}} \pcfunc_{\node} (\bm{w}, \bm{V} \setminus \bm{W}))$}
    \ElseIf{$n$ \emph{is product node}}{\Return $\texttt{PROD}(\{\texttt{MARG}(\node_i; \bm{W})\}_{\node_i \in \inputs(\node)})$}
    \ElseIf{$n$ \emph{is sum node}}{\Return $\texttt{SUM}(\{\texttt{MARG}(\node_i; \bm{W})\}_{\node_i \in \inputs(\node)}, \{\weight_{\node, \node_i}
    \}_{\node_i \in \inputs(\node)})$}
    \KwResult{Smooth and (structured) decomposable probabilistic circuit with root node $\node'$ such that $\pcfunc_{\node'}(\bm{V} \setminus \bm{W}) = \overline{\pcfunc_{\node}}$}
    \caption{Marginalization $\texttt{MARG}(\node; \bm{W})$}
    \label{alg:marg}
\end{algorithm}

\begin{algorithm*}[tbp]
    \KwInput{Compatible probabilistic circuits with root nodes $n^{(1)}, n^{(2)}$}
    \If{$n^{(1)}, n^{(2)}$ \emph{are input nodes}}{\Return $\texttt{INPUT}(\pcfunc_{n^{(1)}} \times \pcfunc_{n^{(2)}})$}
    \ElseIf{$n^{(1)}, n^{(2)}$ \emph {are product nodes}}{\Return $\texttt{PROD}\left(\{\texttt{MULTIPLY}(n^{(1)}_i, n^{(2)}_i)\}_{(n^{(1)}_i, n^{(2)}_i) \in \texttt{SORT-NODES-BY-SCOPE}(\inputs(n^{(1)}), \inputs(n^{(2)}))}\right)$}
    \ElseIf{$n^{(1)}, n^{(2)}$ \emph {are sum nodes}}{\Return $\texttt{SUM}\left(\{\texttt{MULTIPLY}(n^{(1)}_i, n^{(2)}_j)\}_{n^{(1)}_i \in \inputs(n^{(1)}), n^{(2)}_j \in \inputs(n^{(2)})}, \{\weight_{n^{(1)}_i} \weight_{n^{(2)}_j}\}_{n^{(1)}_i \in \inputs(n^{(1)}), n^{(2)}_j \in \inputs(n^{(2)})}\right)$}
    \KwResult{Smooth and structured decomposable probabilistic circuit with root node $\node'$ such that $\pcfunc_{\node'} = \pcfunc_{n^{(1)}} \times \pcfunc_{n^{(2)}}$}
    \caption{Multiplication $\texttt{MULTIPLY}(\node^{(1)}, \node^{(2)})$}
    \label{alg:prod}
\end{algorithm*}

\begin{algorithm}[tbp]
    \KwInput{Augmented probabilistic circuit with root node $n$, over variables $\bm{V}, \bm{U}, \bm{W}$}
    $\node' \gets \texttt{MARG}(\node; \bm{W})$\;
    $\node'' \gets \texttt{CONJ}(\node')$\;
    $\node''' \gets \texttt{MULTIPLY}(\node'', \node')$\;
    \Return $\texttt{MARG}(\node'''; \bm{U})$\;
    \KwResult{InceptionPC}
    \caption{InceptionPC Construction}
    \label{alg:inception}
\end{algorithm}

For completeness, we provide algorithms for (i) the conjugation operation on PCs with complex weights; and (ii) the procedure of constructing materialized \modelshort{}s. For convenience, we will interchangably refer to circuits $\circuit$ and their root node $\node$. We will write $\texttt{SUM}(C, \Theta)$ to denote the operation of constructing a sum node with inputs $C$ and weights $\Theta$, $\texttt{PROD}(C)$ to denote the operation of constructing a product node with inputs $C$, and $\texttt{INPUT}(f)$ to denote the operation of constructing an input node with function $f$.

In Algorithm \ref{alg:conj}, we show the algorithm for conjugation, which simply conjugates the weights at sum nodes and input functions. In Algorithm \ref{alg:inception}, we show the process of constructing materialized \modelshort{}s. This depends on the basic $\texttt{MULTIPLY}$ and $\texttt{MARG}$ on circuits \citep{vergari2021compositional}, which we reproduce in Algorithms \ref{alg:prod}, \ref{alg:marg}. Note that multiplying two circuits while not blowing up the circuit size generally requires them to be \emph{compatible} \citep{vergari2021compositional},\footnote{It was recently shown that incompatible circuits can still sometimes be multiplied efficiently \citep{ZhangAISTATS25}; that is, compatibility is a sufficient but not necesssary condition.} which, intuitively speaking, requires the product nodes in both circuits to decompose their scope in the same way. In Algorithm \ref{alg:marg} we multiply a structured decomposable and smooth circuit and its conjugate, which share the same structure and thus are compatible. 

\section{Log-Sum-Exp Trick with Complex Numbers} \label{apx:logsumexp_complex}

To avoid numerical under/overflow, we perform computations in log-space when computing a forward pass of a PC. For complex numbers, this means keeping the \emph{modulus} of the number in log-space and the \emph{argument} in linear-space. 

Explicitly, given a set of complex numbers $x_1 = e^{u_1 + iv_1}, ..., x_n=e^{u_n + iv_n}$ such that the log-modulus $u_k \in \mathbb{R}$ and argument $v_k \in \mathbb{R}$ are stored in memory, we can compute the log-modulus $u$ and argument $v$ of $x = x_1 + \ldots + x_n$ as follows:
\begin{align}
    u &= \log(\left|e^{u_1 - u_{\max} + iv_1} + \ldots + e^{u_n - u_{\max} + iv_n}\right|) + u_{\max} \\
    v &= \arg(e^{u_1 - u_{\max} + iv_1} + \ldots + e^{u_n - u_{\max} + iv_n})
\end{align}
where $u_{\max} = \max(u_1, .., u_n)$, $|\cdot|$ is the modulus function for complex numbers, and $\arg$ is the principal value of the argument function (i.e. $\in (-\pi, \pi]$).

\section{Experimental Details} \label{apx:experiment_details}

For all experiments, we use the Adam optimizer \citep{kingma2015adam} with learning rate $0.01$. We use a batch size of $64$ for the binary datasets, and $250$ for all image datasets (ImageNet32, ImageNet64). For the binary datasets, we train for 100 epochs and average across 5 runs. For the image datasets, we train for 200 epochs. Model training was performed on NVIDIA A6000 48GB and NVIDIA H100 80GB GPUs.

For the PC input functions, we use categorical inputs for each pixel $\var$. For the binary datasets, this corresponds to 2 parameters for each variable $\pcfunc(\var = i)$ for each $i = 0, 1$, while for the image datasets, for each color channel, we have 256 parameters $\pcfunc(\var = i)$ for each $i = 0, ..., 255$. For monotone PCs, this takes values in $\mathbb{R}^{\geq 0}$, for squared real PCs or real \modelname{}PCs, this takes values in $\mathbb{R}$, and for squared complex PCs or complex \modelname{}PCs this takes values in $\mathbb{C}$. 

\section{Additional Experimental Results for Binary Datasets} \label{apx:experiment_train}

We show in Table \ref{tbl:debd_train} training negative log-likelihoods on the twenty binary datasets, together with the difference compared with the test negative log-likelihoods shown in Table \ref{tbl:debd}. It can be seen that the complex parameterization of Inception PCs almost always achieves the best training log-likelihoods, sometimes with large margins over the other models. This indicates the added expressive efficiency obtained through the InceptionPC architecture and employing complex rather than non-negative or real parameters. However, they also generally overfit by a larger margin, leading to some cases where another model achieves better test-set performance. Preventing this overfitting, perhaps through regularization, is thus an important consideration when applying these models to small datasets.

\begin{table*}[t]
\centering
\scalebox{0.88}{
    
\begin{tabular}{@{}cccccccc@{}}
\toprule
\multicolumn{1}{c}{\textbf{Dataset}} & \multicolumn{7}{c}{\textbf{Model}}                                                                         \\ \cmidrule(l){2-8} 
                                     & Monotone PC & \multicolumn{3}{c}{Squared PC} & \multicolumn{3}{c}{\modelname{} PC}                        \\ 
                                     &            & Non-negative    & Real  & Complex      & Non-negative & Real & Complex \\ \midrule
nltcs                                &   6.03 (-0.01)       &       6.02 (-0.00)       &  6.01 (-0.02)  &     6.00 (-0.03)   &               5.96 (-0.04)         &    6.98 (-0.04)   &               \textbf{5.91} (-0.11)          \\
msnbc                      &       6.25 (-0.00)    &       6.24 (+0.01)       & 6.11 (+0.01)  &    6.07(-0.00)       &       6.06 (+0.01)           &     6.05 (-0.00)    &             \textbf{6.04} (-0.00)         \\
kdd-2k                     &   2.36 (+0.23)        &    2.35 (+0.25)      &  2.49 (+0.26)  &   2.35 (+0.23)        &   2.34 (+0.22)            &     2.36 (+0.23)       &         \textbf{2.32} (+0.20)           \\
plants                        &     13.44 (-0.26)       &        13.14 (-0.24)        & 14.77 (-0.29) &    12.80 (-0.33)       &              12.19 (-0.62)       &    12.64 (-0.42)    &   \textbf{11.81} (-0.95)                  \\ 
jester                        &    52.91 (+0.14)       &        52.74 (+0.18)        &  55.00 (+0.49)   &   52.75 (-0.22)      &                     51.13 (-1.38)        &   51.13 (-1.47)  & \textbf{49.74} (-3.01)                 \\ 
audio                      &     40.22 (-0.05)        &       40.04 (-0.04)   &  41.56 (+0.07) &       39.81 (-0.20)      &                    39.08 (-0.80)  &  39.20 (-0.71)  &   \textbf{38.37} (-1.68)          \\ 
netflix                       &     56.90 (-0.19)     &  56.69 (-0.16)   &   57.59 (-0.09) &    56.36 (-0.34)     &    55.21 (-1.31)   &  55.36 (-1.19)   &  \textbf{54.22} (-2.52)  \\ 
accidents            &   29.42 (-0.15)       &        27.76 (-0.17)   & 28.04 (-0.11)   &   26.84 (-0.21)      &     26.01 (-0.69)    & 27.17 (-0.13)  &   \textbf{25.61} (-1.00)  \\
retail        &     10.85 (-0.14)        &        \textbf{10.68} (-0.14)   & 10.90 (+0.10)  & 10.80 (-0.15)    &   10.86 (-0.14)     & 10.84 (-0.11) &  10.78 (-0.17) \\ 
pumsb-star          &   27.97 (-0.01)   & 24.94 (-0.01)     &   25.75 (+0.06)   &  23.94 (+0.04)  &     23.56 (-0.13)       &      24.90 (+0.05)          & \textbf{ 22.80 }(-0.23)  \\
dna                       &    79.86 (-0.35)       &       79.43 (-0.52)     & 78.80 (-0.35)  &       77.64 (-0.53)       &     78.07 (-1.78)         &  78.00 (-2.11) & \textbf{75.57} (-4.20) \\
kosarek                      &    10.90 (+0.13)        &     10.63 (+0.09)       & 12.24 (+0.21) &   10.65 (+0.06)          &       10.77 (+0.08)        &  10.97 (+0.14) & \textbf{10.56} (-0.04)  \\
msweb                       &     10.34 (-0.10)       &     \textbf{9.83} (-0.09)      & 10.50 (-0.08)  &    10.05 (-0.12)     &       10.75 (-0.09)   & 10.25 (-0.09) & 9.96 (-0.14) \\
book                       &    33.14 (-0.56)        &   32.62 (-0.70)    & 36.52 (-0.50) & 32.42 (-0.53)    &    31.81 (-1.70)        &  33.04 (-1.14) & \textbf{30.69} (-2.98)   \\
   eachmovie                   &    53.75 (+0.92)       &    52.19 (+0.91)     & 63.26 (+1.23)  &      51.14 (-1.19)     &   49.18 (-1.58)    &  48.53 (-2.69) & \textbf{45.29} (-5.12) \\
web-kb                      &     148.34 (-7.00)       &      144.82 (-7.02)     & 155.86 (-6.17) &      142.80 (-12.20)     &  140.60 (-11.14)  & 143.76 (-9.56) &\textbf{ 139.48} (-14.50) \\
reuters-52                     &    103.90 (+8.68)     &    100.82 (+8.19)     & 106.20 (+9.95) &       100.80 (+6.90)      &     99.57 (+6.40)     &  \textbf{98.01} (+8.34) & 98.22 (+4.42) \\
20ng                     &   139.72 (-15.33)       &     137.63 (-15.35)   & 149.20 (-14.99) & 138.17 (-16.62)      &     136.04 (-18.11)     & 140.19 (-15.28) & \textbf{135.42 }(-19.76)  \\
bbc                      &     242.90 (-11.08)      &    239.24 (-11.64)   & 248.00 (-11.04) &      233.90 (-19.23)       &     232.51 (-18.77)      &  234.53 (-18.75) & \textbf{229.84} (-23.53) \\
ad                      &    16.43 (-0.50)     &    14.98 (-0.56) &    15.69 (-0.63)    & 14.65 (-0.70) &     15.25 (-0.77)      &    15.25 (-0.56)      &  \textbf{14.48} (-0.84) \\
\bottomrule
\end{tabular}
    }
    \caption{Train negative log-likelihoods on 20 binary datasets, show together with the difference with the test negative log-likelihood (lower is better). A negative difference indicates that the training NLL is less (better) than the test NLL, while a positive difference indicates the opposite.}
    \label{tbl:debd_train}
\end{table*}

\end{document}